\documentclass{article}

\usepackage{microtype}
\usepackage{graphicx}
\usepackage{subfigure}
\usepackage{booktabs} 

\usepackage{hyperref}

\usepackage[accepted]{icml2019}

\usepackage[utf8]{inputenc}
\usepackage[T1]{fontenc}
\usepackage{amsmath,amsthm,amssymb}
\usepackage{cleveref}
\usepackage{mathtools}
\usepackage{autonum}
\usepackage{xcolor}
\usepackage{tikz}
\usepackage{comment}
\usetikzlibrary{positioning, arrows.meta, decorations.pathmorphing}
\allowdisplaybreaks

\usepackage{latex_abbreviations}

\newcommand{\argmax}{\mathrm{argmax}}

\newcommand{\dom}{\mathrm{dom}}

\newcommand{\favg}{f_\mathrm{avg}}

\newcommand{\supp}{\mathrm{supp}}
\newcommand{\bfzero}{\mathbf{0}}

\newcommand{\non}{\mathrm{non}}
\newcommand{\gap}{\mathsf{GAP}}

\newtheorem{thm}{Theorem}
\newtheorem{lem}{Lemma}
\newtheorem{cor}{Corollary}
\newtheorem{prop}{Proposition}
\theoremstyle{definition}
\newtheorem{defn}{Definition}
\newtheorem{rem}{Remark}
\newtheorem{exmp}{Example}

\icmltitlerunning{Beyond Adaptive Submodularity: Approximation Guarantees of Greedy Policy with Adaptive Submodularity Ratio}

\begin{document}

\twocolumn[
\icmltitle{Beyond Adaptive Submodularity: Approximation Guarantees of\\Greedy Policy with Adaptive Submodularity Ratio}




\begin{icmlauthorlist}
\icmlauthor{Kaito Fujii}{utokyo}
\icmlauthor{Shinsaku Sakaue}{nttcs}
\end{icmlauthorlist}

\icmlaffiliation{utokyo}{University of Tokyo}
\icmlaffiliation{nttcs}{NTT Communication Science Laboratories}

\icmlcorrespondingauthor{Kaito Fujii}{kaito\_fujii@mist.i.u-tokyo.ac.jp}

\icmlkeywords{Machine Learning, ICML}

\vskip 0.3in
]



\printAffiliationsAndNotice{}  

\begin{abstract}
	We propose a new concept named \textit{adaptive submodularity ratio} to study the greedy policy for sequential decision making. 
	While the greedy policy is known to perform well for a wide variety of adaptive stochastic optimization problems in practice, its theoretical properties have been analyzed 
	only for a limited class of problems. 
	We narrow the gap between theory and practice 
	by using adaptive submodularity ratio, which enables us to prove 
	approximation guarantees of the greedy policy 
	for a substantially wider class of problems. 
	Examples of newly analyzed problems include important applications such as adaptive influence maximization and adaptive feature selection. 
	Our adaptive submodularity ratio also provides bounds of \textit{adaptivity gaps}. Experiments confirm that the greedy policy performs well with the applications being considered compared to standard heuristics.
\end{abstract}

\section{Introduction}\label{sec:intro}
	Sequential decision making plays a crucial role in machine learning. 
	In various scenarios, we must design an effective policy that repeatedly decides the next action to be taken by using the feedback obtained so far.
	The greedy policy is a simple but empirically effective approach to sequential decision making. 
	At each step, it myopically makes a decision that seems the most beneficial among feasible choices. 
	
	\textit{Adaptive submodularity} \cite{GK11} is a well-established framework for analyzing greedy algorithms for sequential decision making.
	It extends \textit{submodularity}, which is a diminishing returns property of set functions, to the setting of adaptive decision making.
	This framework has successfully provided theoretical guarantees for greedy algorithms for active learning \cite{GKR10}, recommendation \cite{GKWEM13}, and touch-based localization in robotics \cite{JCKKBS14}.
	
	However, adaptive submodularity is not omnipotent. 
	While the greedy policy works well for various sequential decision making problems, 
	many of these problems do not have adaptive submodularity. 
	In fact, even if an objective function is submodular in the non-adaptive setting, its adaptive version does not always have adaptive submodularity.
	\textit{Adaptive influence maximization} is one such example. 
	In this problem, a decision maker aims at spreading information about a product by selecting several advertisements. 
	She repeatedly alternates between selecting an advertisement and observing its effect.
	The objective function of this problem is known to have adaptive submodularity in the independent cascade model \cite{GK11}, but not in a more general diffusion model called the \textit{triggering model} \cite{KKT03}, which is extensively studied as an important class of diffusion models \cite{LKGFVG07,TXS14}. 
	Note that this objective function satisfies submodularity in the non-adaptive setting, while it does not satisfy adaptive submodularity in the adaptive setting. 
	Examples of other problems lacking 
	adaptive submodularity 
	appear in many applications 
	such as feature selection and active learning. 
	Therefore, we are waiting for an analysis framework that goes beyond 
	adaptive submodularity. 
	
	\begin{table*}\label{table:our_results}
		\caption{Summary of our theoretical results about adaptive bipartite influence maximization and adaptive feature selection. We show lower bounds for the adaptive submodularity ratios, the approximation ratios of the adaptive greedy algorithm, and the adaptivity gaps. Let $\lambda_{\min, \ell} = \min_{\phi} \min_{S \subseteq V \colon |S| \le \ell} \lambda_{\min} (\bfA(\phi)_S^\top \bfA(\phi)_S)$ and $\lambda_{\max, \ell} = \max_{\phi} \max_{S \subseteq V \colon |S| \le \ell} \lambda_{\max} (\bfA(\phi)_S^\top \bfA(\phi)_S)$. Parameters $q$ and $d$ are determined by the diffusion model and the underlying graph structure. The results of \citep{GK11} are indicated by $\dagger$.}
		\vskip 0.15in
		\centering
		\begin{tabular}{lccc}
			\toprule
			Problem & Adaptive submodularity ratio & Adaptive greedy & Adaptivity gaps\\
			\midrule
			Linear threshold & $(k+1) / 2k$ & $1-\exp(-(k+1) /2k)$ & $(k+1) / 2k$\\
			Independent cascade & $1^\dagger$ & $1 - 1 / \rme^\dagger$ & $(1 - q)^{\min\{d, k\} - 1}$\\
			Triggering & $(k+1) / 2k$ & $1-\exp(- (k+1) / 2k)$ & \\
			Feature selection & $\lambda_{\min, k + \ell}$ & $1 - \exp(- \lambda_{\min, k + \ell})$ & $\lambda_{\min, k} / \lambda_{\max, k}$\\
			\bottomrule
		\end{tabular}
	\end{table*}
	
	In the non-adaptive setting,  
	{\it submodularity ratio} 
	\cite{Das2011} is a prevalent tool for handling non-submodular functions \cite{KEDNG17,EDFK17}.  
	Intuitively, it is a parameter of monotone set functions that measures their distance to submodular functions.
	An adaptive variant of submodularity ratio 
	would be a promising approach to handling 
	functions that lack adaptive submodularity, 
	but how to define it is quite non-trivial since 
	there is a large discrepancy between 
	the non-adaptive and adaptive settings 
	as exemplified above. 
	In particular, success in 
	defining 
	an adaptive version of submodularity ratio 
	involves meeting the following two requirements: 
	it must yield an approximation guarantee of the greedy policy, 
	and it must be bounded in various important applications 
	such as the adaptive influence maximization 
	and adaptive feature selection. 
	Previous works \cite{Kusner14,YGO17} 
	tried to define similar notions, 
	but none of them meet the requirements.

	\paragraph{Our Contribution.}
	We propose an analysis framework, {\it adaptive submodularity ratio}, that meets the aforementioned requirements. 
	An advantage of our proposal is that it has the potential to yield various theoretical results as in Table~\ref{table:our_results}. 
	Below we summarize our main contributions. 
	\begin{itemize}
		\item We propose the definition of the adaptive submodularity ratio and, by using it, we prove an approximation guarantee of the adaptive greedy algorithm. 
		\item We give a bound on the \textit{adaptivity gap}\footnote{The adaptivity gap is a different concept from \textit{adaptive complexity} \cite{BS18}.}, 
		which represents the superiority of adaptive policies 
		over non-adaptive policies, through the lens of the adaptive submodularity ratio.
		\item We provide lower-bounds of adaptive submodularity ratio for two important applications: 
		adaptive influence maximization on bipartite graphs in the triggering model and adaptive feature selection. 
		Regarding the former one, 
		we show that our result is tight. 
		\item Experiments confirm that the greedy policy 
		performs well for the considered applications. 
	\end{itemize}
	
	\paragraph{Organization.} 
	The rest of this paper is organized as follows.
	\Cref{sec:pre} provides the basic concepts and definitions.
	In \Cref{sec:ratio}, we formally define the adaptive submodularity ratio, which is the key concept of this study.
	In \Cref{sec:greedy,sec:gap}, we provide bounds on the approximation ratio of the adaptive greedy algorithm and adaptivity gaps, respectively, by using the adaptive submodularity ratio.
	In \Cref{sec:infmax,sec:feature}, we apply the frameworks developed in \Cref{sec:greedy,sec:gap} to two applications: adaptive influence maximization and adaptive feature selection.
	In \Cref{sec:experiment}, we experimentally check the performance of the adaptive greedy algorithm in several applications.
	In \Cref{sec:related} we review related work. 
	
	\section{Preliminaries}\label{sec:pre}
	\paragraph{Adaptive Stochastic Optimization.}
	Adaptive stochastic optimization is a general framework for handling problems of sequentially selecting elements, 
	where we can observe the states of only the selected elements. 
	Let $V$ be the ground set consisting of a finite number of elements.
	Suppose every element $v \in V$ is assigned to some state in $\calY$, which is the set of all possible states.
	We let $\phi \colon V \to \calY$ be a map 
	that 
	associates each element, $v\in V$, 
	with a state, $\phi(v) \in \calY$. 
	We consider the Bayesian setting where $\phi$ is generated from a known prior distribution $p(\phi)$.
	Let $\Phi$ be a random variable representing the randomness of the realization $\phi$.
	
	A decision maker can select one element $v \in V$ at each step.
	After selecting $v$, she can observe the state $\phi(v)$ of $v$.
	She repeatedly selects an element and then observes its state. 
	The important point is that she can 
	utilize the information about the states observed so far for selecting the next element. 
	We denote by $\psi = \{(v_1, \phi(v_1)),\dots, (v_\ell, \phi(v_\ell))\}$ the partial realization observed so far, 
	where $\{v_1,\dots,v_\ell\}$ is the set of selected elements.
	The decision maker's strategy can be described as a \textit{policy tree}, or simply \textit{policy}.
	A policy is a decision tree that determines the element to 
	be selected next.
	Formally, 
	a policy $\pi$ is a partial map that returns an element $v \in V$ to be selected next given partial realization $\psi$ observed so far.
	
	The goal of the decision maker is to maximize the expected value of the objective function $f \colon 2^V \times \calY^V\to\bbR$. 
	The objective function value $f(S, \phi)$ depends on the set $S$ of selected elements and the states $\phi$ of all elements.
	At the beginning, she does not know $\phi$, 
	but she can get partial information of $\phi$ by observing 
	state $\phi(v)$ of selected $v$.
	In parallel, she must select elements to construct $S$ that has high utility under the realization $\phi$.
	Let $E(\pi, \phi)\subseteq V$ be the set selected by policy $\pi$ under realization $\phi$.
	The expected value achieved by policy $\pi$ is
	\begin{equation}
	\favg(\pi) = \bbE_\Phi[f(E(\pi, \Phi), \Phi)],
	\end{equation}
	where the expectation is taken with regard to the random variable $\Phi$ generated from $p$.

	\paragraph{Adaptive Submodularity and Adaptive Monotonicity.} Adaptive submodularity, which is an adaptive extension of submodularity, is a diminishing returns property of the expected marginal gain. The expected marginal gain of $v \in V$ when $\psi$ has been observed so far is defined as 
	\begin{align}
		&\Delta(v | \psi)\\
		&\coloneqq \bbE [f(\dom(\psi) \cup \{v\}, \Phi) - f(\dom(\psi), \Phi) | \Phi \sim \psi], 
	\end{align} 
	where $\dom(\psi) \coloneqq \{ v \in V \mid \exists y \in \calY , ~ (v, y) \in \psi \}$.
	We write $\Phi \sim \psi$ if $\Phi$ is generated from the posterior distribution $p(\phi|\psi)$. 
	Given current realization $\psi$, 
	the expected marginal gain, $\Delta(v | \psi)$, 
	represents the expected increase in 
	the objective value 
	yielded by selecting $v$. 
	Adaptive submodularity is defined as follows:
	\begin{defn}[Adaptive submodularity \cite{GK11}]
		Let $f \colon 2^V \times \calY^V \to \bbR$ be a set function and $p$ a distribution of $\phi$.
		We say $f$ is adaptive 
		submodular with respect to $p$ if for any partial realization $\psi \subseteq \psi'$ and any element $v \in V \setminus \dom(\psi')$, it holds that 
		\begin{equation}
		\Delta(v | \psi) \ge \Delta(v | \psi').
		\end{equation}
	\end{defn}
	
	The monotonicity can also be extended to the adaptive setting as follows:
	\begin{defn}[Adaptive monotonicity \cite{GK11}]
		Let $f \colon 2^V \times \calY^V \to \bbR$ be a set function and $p$ a distribution of $\phi$. We say $f$ is adaptive monotone with respect to $p$ if for any partial realization $\psi$ and any element $v \in V \setminus \dom(\psi)$, it holds that 
		\begin{equation}
		\Delta(v | \psi) \ge 0.
		\end{equation}
	\end{defn}
	
	\paragraph{Other Notations for Adaptive Stochastic Optimization.}
	The expected marginal gain of policy $\pi$ with partial realization $\psi$ is defined as
	\begin{align}
	&\Delta(\pi | \psi) 
	\\
	&\coloneqq
	\bbE [f(\dom(\psi) \cup E(\pi, \Phi), \Phi) - f(\dom(\psi), \Phi) | \Phi \sim \psi].
	\end{align}
	Similarly, the expected marginal gain of set $S \subseteq V$ with partial realization $\psi$ is defined as
	\begin{equation}
	\Delta(S | \psi) \coloneqq \bbE [f(\dom(\psi) \cup S, \Phi) - f(\dom(\psi), \Phi) | \Phi \sim \psi].
	\end{equation}
	Let $\Pi_k \coloneqq \{\pi \mid \forall \phi, ~ |E(\pi, \phi)| \le k \}$ be the set of all policies whose heights do not exceed $k$.

\paragraph{Submodularity Ratio and Supermodularity Ratio.}
The submodularity ratio of a monotone non-negative set function $f \colon 2^V \to \bbR_{\ge 0}$ with respect to set $U \subseteq V$ and parameter $k \ge 1$ is defined to be
\begin{equation}
	\gamma_{U, k}(f) \coloneqq \min_{L \subseteq U, ~ S \colon |S| \le k} \frac{\sum_{v \in S} f(v | L)}{f(S | L)}, 
\end{equation}
where $f(v | L) \coloneqq f(L \cup \{v\}) - f(L)$ and $f(S | L) \coloneqq f(L \cup S) - f(L)$. 
If the numerator and denominator are both $0$, the submodularity ratio is considered to be 1. 
We have $\gamma_{U, k}\in[0,1]$,  
and a monotone set function $f$ is submodular if and only if $\gamma_{U, k} = 1$ for every $U \subseteq V$ and $k \ge 1$.

As an opposite concept of the submodularity ratio, the \textit{supermodularity ratio}, 
was considered in \citet{Bogunovic18}, 
which is defined as follows:
\begin{equation}
	\beta_{U, k}(f) \coloneqq \min_{L \subseteq U, ~ S \colon |S| \le k} \frac{f(S | L)}{\sum_{v \in S} f(v | L)}, 
\end{equation} 
where we regard $0/0=1$. 
We have $\beta_{U, k}\in[1/k,1]$, 
and $f$ is supermodular if and only if 
$\beta_{U, k} = 1$ for every $U \subseteq V$ and $k \ge 1$.
We omit $f$ from $\gamma_{U, k}(f)$ and $\beta_{U, k}(f)$ if it is clear from the context.

\section{Adaptive Submodularity Ratio}\label{sec:ratio}
In this section, we provide a precise definition of the adaptive submodularity ratio, which extends the submodularity ratio from the non-adaptive setting to the adaptive setting. 
We need to define it carefully so that it can yield 
an approximation guarantee of the greedy policy. 
An important point is 
to generalize subset $S$ of size at most $k$, 
used to define the submodularity ratio, 
to policy $\pi$ of height at most $k$.
\begin{defn}[Adaptive submodularity ratio]
	Suppose that $f \colon 2^V \times \calY^V \to \bbR$ is adaptive monotone w.r.t.\ a distribution $p$.
	Adaptive submodularity ratio $\gamma_{\psi, k} \in [0, 1]$ of $f$ and $p$ with respect to partial realization $\psi$ and parameter $k \in \bbZ_{\ge 0}$ is defined to be
	\begin{align}
		&\gamma_{\psi, k} (f, p) \coloneqq 
	\\ 
	&\min_{
		\psi' \subseteq \psi, ~ \pi \in \Pi_k
	}
		\frac{\sum_{v \in V} \Pr(v \in E(\pi, \Phi) | \Phi \sim \psi') \Delta(v | \psi') }{\Delta(\pi | \psi')}.
	\end{align}
	We omit $f$ and $p$ if they are clear from the context. 
	We also define $\gamma_{\ell, k} \coloneqq \min_{\psi : |\psi| \le \ell} \gamma_{\psi, k}$. 
\end{defn}
Intuitively, the adaptive submodularity ratio indicates 
the distance between $(f, p)$ and the class of adaptive submodular functions.
As with the non-adaptive setting, 
$\gamma_{\psi, k}(f, p) = 1$ implies the adaptive submodularity of $f$, 
which can formally be written as follows:  
\begin{prop}\label{prop:ratio_one}
	It holds that $\gamma_{\psi, k}(f, p) = 1$ for any partial realization $\psi$ and $k \in \bbZ_{\ge 0}$ if and only if $f$ is adaptive submodular with respect to $p$.
\end{prop}
The proof is given in \Cref{sec:app-ratio}.

\section{Adaptive Greedy Algorithm}\label{sec:greedy}
In this section, we present a new approximation ratio guarantee for the adaptive greedy algorithm based on the adaptive submodularity ratio. 
Thanks to this result, 
once the adaptive submodularity ratio is bounded, 
we can obtain approximation guarantees of the adaptive greedy algorithm for various applications.  
The adaptive greedy algorithm is an algorithm that starts with an empty set and repeatedly selects the element with the largest expected marginal gain. The detailed description is given in \Cref{alg:adaptive-greedy}.
\citet{GK11} have shown that this algorithm achieves $(1 - 1/\rme)$-approximation to the expected objective value of an optimal policy if $f$ is adaptive submodular w.r.t.\ $p$.
Here we extend their result and show that the adaptive greedy algorithm achieves $(1 - \exp( - \gamma_{\ell, \ell} ))$-approximation, where $\ell$ is the number of selected elements. 
More precisely, 
we can bound the approximation ratio 
relative to any policy $\pi^*$ of height $k$ as follows: 
\renewcommand{\algorithmicrequire}{\textbf{Input}}
\renewcommand{\algorithmicensure}{\textbf{Output}}
\begin{algorithm}[t]
	\caption{Adaptive greedy algorithm \cite{GK11}}
	\label{alg:adaptive-greedy}
	\begin{algorithmic}[1]
		\REQUIRE{The value oracle for the expected marginal gain $\Delta( \cdot | \cdot)$ associated with $f \colon 2^V \times \calY^V$ and $p \in \triangle^{\calY^V}$, a cardinality constraint $\ell \in \bbZ_{\ge 0}$.}
		\ENSURE{$\psi_\ell$ a set of observations of size $\ell$.}
		\STATE $\psi_0 \gets \emptyset$.
		\FOR{$i = 1,\dots, \ell$}
		\STATE $v \gets \argmax_{v \in V} \Delta(v | \psi_{i-1})$.
		\STATE Observe $\phi(v)$ and let $\psi_i \gets \psi_{i-1} \cup \{(v, \phi(v))\}$.
		\ENDFOR
		\STATE \textbf{return} $\psi_\ell$.
	\end{algorithmic}
\end{algorithm}

\begin{thm}\label{thm:adaptive-greedy}
	Suppose $f \colon 2^V \times \calY^V \to \bbR_{\ge 0}$ is adaptive monotone with respect to $p$.
	Let $\pi$ be a policy representing the adaptive greedy algorithm until $\ell$ step. 
	Then, for any policy $\pi^* \in \Pi_k$, it holds that
	\begin{equation}
	\favg(\pi) \ge \left(1 - \exp\left(- \frac{\gamma_{\ell, k} \ell}{k} \right) \right) \favg(\pi^*), 
	\end{equation}
	where $\gamma_{\ell, k}$ is the adaptive submodularity ratio of $f$ w.r.t.\ $p$.
\end{thm}
We provide the proof in \Cref{sec:app-greedy}.

\section{Non-adaptive Policies and Adaptivity Gaps}\label{sec:gap}
We show that the adaptive submodularity ratio is also 
useful for theoretically comparing the 
performances of adaptive and non-adaptive policies. 
More precisely, we present a lower-bound of 
the \textit{adaptivity gap}, 
which represents the performance gap between 
adaptive and non-adaptive polices, 
by using the adaptive submodularity ratio. 
The adaptivity gap is defined as follows:  
\begin{defn}[{Adaptivity gaps}]
	The adaptivity gap $\gap_k(f,p)$ of an objective function $f \colon 2^V \times \calY^V \to \bbR_{\ge 0}$ and a probability distribution $p$ of $\phi \colon V \to \calY$ is defined as 
	the ratio between an optimal adaptive policy and an optimal non-adaptive policy, i.e.,
	\begin{equation}
		\gap_k(f, p) = \frac{\max_{M \colon |M| \le k} \bbE_\Phi [ f(M, \Phi) ]}{\max_{\pi^* \in \Pi_k} \favg(\pi^*)}, 
	\end{equation}
	where $k$ is the height of adaptive and non-adaptive policies.
\end{defn}

\begin{thm}\label{thm:gap}
	Let $f \colon 2^V \times \calY^V \to \bbR_{\ge 0}$ be an objective function and $p$
	a probability distribution of $\phi \colon V \to \calY$.
	Let $\gamma_{\emptyset, k}$ be the adaptive submodularity ratio of $f$ w.r.t.\ $p$.
	Let $\beta_{\emptyset, k}$ be the supermodularity ratio of the set function $\bbE_\Phi [ f(\cdot, \Phi) ]$ of non-adaptive policies. 
	We have
	\begin{equation}
	\gap_k(f, p) \ge \beta_{\emptyset, k} \gamma_{\emptyset, k}.
	\end{equation}
\end{thm}

Therefore, 
given any non-adaptive $\alpha$-approximation algorithm, 
we can evaluate its performance relative to an optimal adaptive 
policy as follows: 
\begin{cor}\label{cor:nonadaptive-approx}
	Let $\pi_\non \in \Pi_k$ be a non-adaptive policy that achieves $\alpha$-approximation to an optimal non-adaptive policy $\pi^*_\non$. Let $\gamma_{\emptyset, k}$ be the adaptive submodularity ratio of $f$ w.r.t.\ $p$. Let $\beta_{\emptyset, k}$ be the supermodularity ratio of the non-adaptive objective function $\bbE_\Phi [ f(\cdot, \Phi) ]$. Let $\pi^*$ be an optimal adaptive policy. We have
	\begin{equation}
	\favg(\pi_\non) \ge \alpha \beta_{\emptyset, k} \gamma_{\emptyset, k} \favg(\pi^*). 
	\end{equation}
\end{cor}
Proofs are given in \Cref{sec:app-gap}.

\section{Adaptive Influence Maximization}\label{sec:infmax}
In this section, we consider adaptive influence maximization on bipartite graphs. 
We provide a bound on the adaptive submodularity ratio 
in the case of the triggering model, 
and we show that this result is tight. 
We also present bounds on the adaptivity gaps in the case of the independent cascade and linear threshold models by using the adaptive submodularity ratio.

Let $G = (V \cup U, A)$ be a directed bipartite graph with source vertices $V$, sink vertices $U$, and directed edges $A \subseteq V \times U$.
In the case of bipartite influence model \cite{Alon12}, this graph represents the relationship between advertisements $V$ and customers $U$.
We consider the problem of selecting several advertisements $S \subseteq V$ to make as much influence as possible on the customers. 
Here, each edge is determined to be alive or dead according to a certain distribution, 
and influence can be spread only through live edges.
Given vertex weights $w \colon U\to \bbR_{\ge 0}$, 
the objective function to be maximized is $f(X) = \sum_{u \in \bigcup_{v \in X} R(v)} w(u)$, 
where, for each $v\in V$, 
$R(v) \subseteq U$ represents a set of vertices that are reachable from $v$ by going through only live edges. 
In the adaptive version of influence maximization, 
at each step, we select a vertex $v \in V$ and observe the states of all outgoing edges $(v, u) \in A$,
while, 
in the non-adaptive setting, 
we select $S \subseteq V$ before observing the states of any edges. 

We consider a general diffusion model called the \textit{triggering model} \cite{KKT03}, which includes various important models such as the independent cascade model and the linear threshold model as special cases.
In the triggering model, each vertex $v \in V$ is associated with some known probability distribution over the power set of incoming edges.
According to this distribution, a subset of incoming live edges is determined. 
A vertex gets activated if and only if it is reachable from 
some selected vertex (or seed vertex) through only live edges. 
We aim to maximize the total weight of activated vertices by appropriately selecting seed vertices.
Note that this objective function is submodular 
in the non-adaptive setting.

For later use, we explain the linear threshold model, a special case of the triggering model.
In this model, 
the probability distribution on the incoming edges of each vertex 
is restricted 
so that 
each vertex has at most one live edge in any realization.
In other words, there exists $b \colon A \to \bbR_{\ge 0}$ such that, 
for each $v \in V$, 
we have $\sum_{a \in \delta_-(v)} b(a) \le 1$, 
where $\delta_-(v)$ is the full set of edges pointing to $v$, 
and $a \in A$ is alive with probability $b(a)$ 
exclusively over $\delta_-(v)$. 
In contrast to the linear threshold model, 
the triggering model accepts any distribution over the power set of $\delta_-(v)$. 

\subsection{Bound of Adaptive Submodularity Ratio}
We first present the bound of adaptive submodularity ratio. 
Here we provide a proof sketch, and the full proof is given in \Cref{sec:app-infmax-linear,sec:app-infmax-triggering}.

\begin{thm}\label{thm:bipartite_ratio}
	Let $G$ be an arbitrary directed bipartite graph and $w$ be any weight function. For any $k \in \bbZ_{\ge 0}$ and partial realization $\psi$, the adaptive submodularity ratio $\gamma_{\psi, k}$ of the objective function and the distribution of the adaptive influence maximization in the triggering model is lower-bounded as follows: 
	\begin{equation}
	\gamma_{\psi, k} \ge \frac{k+1}{2k}. 
	\end{equation}
\end{thm}

\begin{proof}[Proof sketch of \Cref{thm:bipartite_ratio}]
Since the objective function and the probability distribution of edge states can be decomposed into those defined for each vertex $u \in U$, it is sufficient to consider the case where $|U| = 1$.

Our goal is to prove
\begin{equation}\label{eq:influence-goal-inequality-sketch}
\begin{aligned}
&\Delta(\pi | \psi')  
\\
&\le \frac{2k}{k+1} \sum_{v \in V} \Pr(v \in E(\pi, \Phi) | \Phi \sim \psi') \Delta(v | \psi')
\end{aligned}
\end{equation}
for any observation $\psi'$ and policy $\pi \in \Pi_k$. 
By duplicating $v \in V$ that appears multiple times in policy tree $\pi$, we can write the above inequality as 
\begin{equation}
	\sum_{v \in V}
	{\rm P}_{v,\pi}
	\left( \frac{2k}{k+1} \Delta(v | \psi') -  \Delta(v | \psi' \cup \psi_v)\right) 
	\ge 0, 
\end{equation}
where 
${\rm P}_{v,\pi}$ is a shorthand for 
$\Pr(v \in E(\pi, \Phi) | \Phi \sim \psi')$ 
and $\psi_v$ is the observation just before $v$ is selected.
We decompose the policy tree into the path wherein $u$ remains inactive and the rest, and prove the inequality for each part separately.
\end{proof}

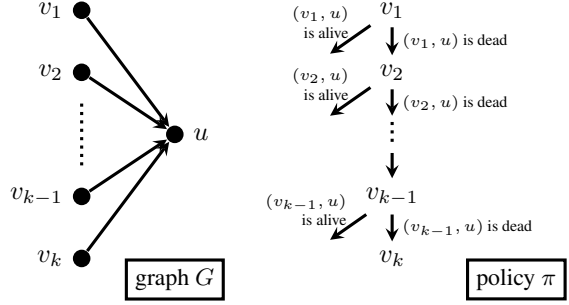
\begin{figure}[t]
\begin{center}
	\begin{tikzpicture}
		\tikzstyle{vertex}=[circle, fill, inner sep=0.01\hsize]
		\tikzstyle{arc}=[-{stealth}, line width=0.005\hsize]
		\node[vertex, label=left:$v_1$] (v1) at (0.0\hsize, 0.4\hsize) {};
		\node[vertex, label=left:$v_2$] (v2) at (0.0\hsize, 0.3\hsize) {};
		\node[vertex, label=left:$v_{k-1}$] (v4) at (0.0\hsize, 0.1\hsize) {};
		\node[vertex, label=left:$v_{k}$] (v5) at (0.0\hsize, 0.0\hsize) {};
		\node[vertex, label=right:$u$] (u1) at (0.15\hsize, 0.2\hsize) {};
		\draw[arc] (v1) -- (u1);
		\draw[arc] (v2) -- (u1);
		\draw[dotted, line width=0.005\hsize] (0.0\hsize, 0.15\hsize) -- (0.0\hsize, 0.25\hsize);
		\draw[arc] (v4) -- (u1);
		\draw[arc] (v5) -- (u1);

		\tikzstyle{policynode}=[]
		\tikzstyle{policyedge}=[-{stealth}, line width=0.005\hsize]
		\node[policynode] (pv1) at (0.5\hsize, 0.4\hsize) {$v_1$};
		\node[policynode] (pv2) at (0.5\hsize, 0.3\hsize) {$v_2$};
		\node[policynode] (pv4) at (0.5\hsize, 0.1\hsize) {$v_{k-1}$};
		\node[policynode] (pv5) at (0.5\hsize, 0.0\hsize) {$v_{k}$};
		\coordinate (pva) at (0.5\hsize, 0.23\hsize);
		\coordinate (pvb) at (0.5\hsize, 0.17\hsize);
		\coordinate (pf2) at (0.4\hsize, 0.33\hsize);
		\coordinate (pf3) at (0.4\hsize, 0.23\hsize);
		\coordinate (pf5) at (0.4\hsize, 0.03\hsize);
		\draw[arc] (pv1) to node [above, anchor=south east, align=right, inner sep=0pt] {\tiny $(v_1,u)$\\[-1ex]\tiny is alive} (pf2);
		\draw[arc] (pv2) to node [above, anchor=south east, align=right, inner sep=0pt] {\tiny $(v_2,u)$\\[-1ex]\tiny is alive} (pf3);
		\draw[arc] (pv4) to node [above, anchor=south east, align=right, inner sep=0pt] {\tiny $(v_{k-1},u)$\\[-1ex]\tiny is alive} (pf5);
		\draw[arc] (pv1) to node [right] {\tiny $(v_1,u)$ is dead} (pv2);
		\draw[arc] (pv2) to node [right] {\tiny $(v_2,u)$ is dead} (pva);
		\draw[dotted, line width=0.005\hsize] ([yshift=-0.01\hsize]pva) -- ([yshift=+0.01\hsize]pvb);
		\draw[arc] (pvb) -- (pv4);
		\draw[arc] (pv4) to node [right] {\tiny $(v_{k-1},u)$ is dead} (pv5);

		\node[anchor=north, draw=black, line width=0.005\hsize] at (0.15\hsize, 0.0\hsize) {\small graph $G$};
		\node[anchor=north, draw=black, line width=0.005\hsize] at (0.7\hsize, 0.0\hsize) {\small policy $\pi$};
	\end{tikzpicture}
	\caption{An example that implies the tightness of our bound.}\label{fig:infmax_instance1}
\end{center}
\end{figure}

We can see that the above bound is tight even for the linear threshold model by considering the following example. 
\begin{exmp}\label{exmp:infmax-star}
    Let $G$ be a bipartite directed graph with $V = \{v_1,\dots, v_k\}$, $U=\{u\}$, and $A = \{ (v_i, u) \mid i \in [k] \}$.
    Let $w$ be the vertex weight such that $w(u) = 1$.
    We consider the linear threshold model in which an edge selected out of $A$ uniformly at random is alive and the other edges are dead.
    We consider a simple policy $\pi$ that selects all vertices one by one until $u$ is activated.
	These graph and policy are illustrated in \Cref{fig:infmax_instance1}.
    Since $\pi$ finally activates $u$, the expected gain of $\pi$ is $\Delta(\pi|\emptyset) = 1$.
    The probability that $\pi$ selects each vertex is $\Pr(v_i \in E(\pi, \Phi)) = (k-i+1)/k$.
    The expected marginal gain of $v_i$ is $\Delta(v_i | \emptyset) = 1/k$.
    The adaptive submodularity ratio can be upper-bounded as
\begin{align}
    \gamma_{\emptyset, k} &\le \frac{\sum_{v \in V} \Pr(v \in E(\pi, \Phi) ) \Delta(v | \emptyset) }{\Delta(\pi | \emptyset)}\\
    &\le \sum_{i=1}^k \frac{k-i+1}{k} \cdot \frac{1}{k}\\
    &\le \frac{k+1}{2k}. 
\end{align}
    Hence the lower-bound in \Cref{thm:bipartite_ratio} is tight. 
\end{exmp}

The assumption that $G$ is bipartite, 
considered in \Cref{thm:bipartite_ratio}, 
may seem excessively strong, 
but it is actually a vital assumption. 
We show that, 
if $G$ is not a bipartite graph, 
the adaptive submodularity ratio can be arbitrarily small; 
in fact, such an example can be constructed 
with the linear threshold model on a very simple graph $G$. 
We describe the details in \Cref{sec:app-infmax-general}.

\subsection{Bound of Adaptivity Gap}
Next we provide a bound on the adaptivity gaps of bipartite influence maximization problems by using the adaptive submodularity ratio. 
First we consider the independent cascade model.
Since the adaptive submodularity holds for the independent cascade model \cite{GK11}, the adaptive submodularity ratio of its objective function is $1$ by \Cref{prop:ratio_one}.
In addition, by using a bound of the curvature \cite{Maehara17} and an inequality between the supermodularity ratio and the curvature \cite{Bogunovic18}, we obtain $\beta_{\emptyset, k}\ge (1 - q)^{\min\{k, d\}-1}$, where $q$ is an upper bound of the probability that each edge is alive and $d$ is the largest degree of the vertex in $V$.
From \Cref{thm:gap}, we obtain the following result.
\begin{prop}
	Let $f$ be the objective function and $p$ the probability distribution of bipartite influence maximization in the independent cascade model. We have 
	\begin{equation}
		\gap_k(f, p) \ge (1 - q)^{\min\{k, d\}-1}.
	\end{equation}
\end{prop}

We can derive a similar bound for the linear threshold model.
Since the expected objective function is a linear function, its supermodularity ratio is $1$.
As a special case of \Cref{thm:bipartite_ratio}, we have $\gamma_{\emptyset, k} \ge \frac{k+1}{2k}$.
Combining these bounds with \Cref{thm:gap}, we obtain the following result. 
\begin{prop}
	Let $f$ be the objective function and $p$ the probability distribution of bipartite influence maximization in the linear threshold model. We have 
	\begin{equation}
		\gap_k(f, p) \ge \frac{k+1}{2k}.
	\end{equation}
\end{prop}

\section{Adaptive Feature Selection}\label{sec:feature}
In this section, we consider an adaptive variant of feature selection for sparse regression. 
All proofs related to this section are presented in \Cref{sec:app-feature-ratio}. 

Let us consider the following scenario. A learner has all feature vectors in advance, but they are not accurate due to sensing noise.
Here each sensor corresponds to a single feature vector.
The learner can obtain accurate feature vectors by replacing inaccurate sensors with high-quality sensors, 
but the number of high-quality sensors is limited to $k$. 
The learner selects $k$ features for observing their accurate feature vectors.

We formalize this scenario as the following problem.
At the beginning, a learner knows a response vector $\bfb \in \bbR^{m}$ and a prior distribution over the features, but does not know the features themselves.
Namely, we regard the inaccurate feature vectors obtained with noisy sensors as prior distributions on accurate feature vectors.
A random variable $\Phi$ indicates the uncertainty over the observed feature vectors.
From the noisy sensors, we can know only a prior distribution of $\Phi$ but not the true $\phi$.
Let $V = [n]$ be the set of features.
At each step, the learner can query a feature $v \in V$ and observe its feature vector $\phi(v) \in \bbR^m$.
We assume the noise of sensors are independent of each other; i.e., there exists a distribution $p_v(\phi(v))$ for each $v \in V$ and we can factorize $p$ as $p(\phi) = \prod_{v \in V} p_v(\phi(v))$.

Let $\bfA(\phi) = (\phi(1) \cdots \phi(n))$ be the realized feature matrix under realization $\phi$.
The objective function to be maximized is defined as 
$
f(S, \phi) = \| \bfb \|^2_2 - \min_{\bfw \in \bbR^S} \| \bfb - \bfA(\phi)_S \bfw \|^2_2.
$

\subsection{Bound of Adaptive Submodularity Ratio}
To bound the adaptive submodularity ratio of adaptive feature selection, we give a general lower bound of the adaptive submodularity ratio by using (non-adaptive) submodularity ratios of all realizations.
\begin{thm}\label{thm:stochastic-ratio}
	Let $f \colon 2^V \times \calY^V \to \bbR$ be adaptive monotone w.r.t.\ distribution $p(\phi)$.
	Assume the value of $f(S, \phi)$ depends only on $(\phi(v))_{v \in S}$ not on $(\phi(v))_{v \in V \setminus S}$, i.e., $f(S, \phi) = f(S, \phi')$ for all $\phi$ and $\phi'$ such that $\phi(v) = \phi(v)$ for all $v \in S$.
	We also assume $p(\phi)$ can be factorized to distributions $p_v(\phi(v))$ of states of each $v \in V$, i.e., $p(\phi) = \prod_{v \in V} p_v(\phi(v))$.
	Let $\gamma^\phi_{X, k}$ be the submodularity ratio of $f( \cdot, \phi)$ for each realization $\phi$.
	For any distribution $p_v$ of $\phi(v)$, the adaptive submodularity ratio $\gamma_{\psi, k}$ can be bounded as
	\begin{equation}
	\gamma_{\psi, k} \ge \min_{\phi \sim \psi} \gamma^\phi_{\dom(\psi), k}.
	\end{equation}
\end{thm}

By using \Cref{thm:stochastic-ratio}
and 
the result of \cite{Das2011}, 
we obtain the following lower bound of the adaptive submodularity ratio.
\begin{cor}\label{cor:stochastic} 
	Assume each column of $\bfA(\phi)$ is normalized. 
	For any $\bfb \in \bbR^n$ and any distribution $p_v$ of each $\phi(v)$, the adaptive submodularity ratio $\gamma_{\ell, k}$ can be bounded as
	\begin{equation}
	\gamma_{\ell, k} \ge \min_{\phi} \min_{S \subseteq V : |S| \le k + \ell} \lambda_{\min}(\bfA(\phi)_S^\top \bfA(\phi)_S), 
	\end{equation}
	where $\lambda_{\min}(\cdot)$ represents the smallest eigenvalue.
\end{cor}

\subsection{Bound of Adaptivity Gap}
We can also obtain a bound on the adaptivity gap of adaptive feature selection as follows:
\begin{prop}\label{prop:feature-gap}
	Let $f(S, \phi) = \| \bfb \|^2_2 - \min_{\bfw \in \bbR^S} \| \bfb - \bfA(\phi)_S \bfw \|^2_2$ and suppose that $p(\phi)$ can be factorized as $p(\phi) = \prod_{v \in V} p_v(\phi(v))$. We have
	\begin{equation}
		\gap_k \ge \frac{\min_{\phi} \min_{S \subseteq V \colon |S| \le k} \lambda_{\min}(\bfA(\phi)_S^\top \bfA(\phi)_S)}{ \max_{\phi} \max_{S \subseteq V \colon |S| \le k} \lambda_{\max}(\bfA(\phi)_S^\top \bfA(\phi)_S)}.
	\end{equation}
\end{prop}

\begin{rem}
	These results on the adaptive submodularity ratio and adaptivity gap can be extended to more general loss functions with restricted strong concavity and restricted smoothness as in \citet{Elenberg18}.
\end{rem}

\begin{figure*}
	\centering
	\subfigure[infmax, synth., linear threshold]{
		\includegraphics[width=0.29\textwidth]{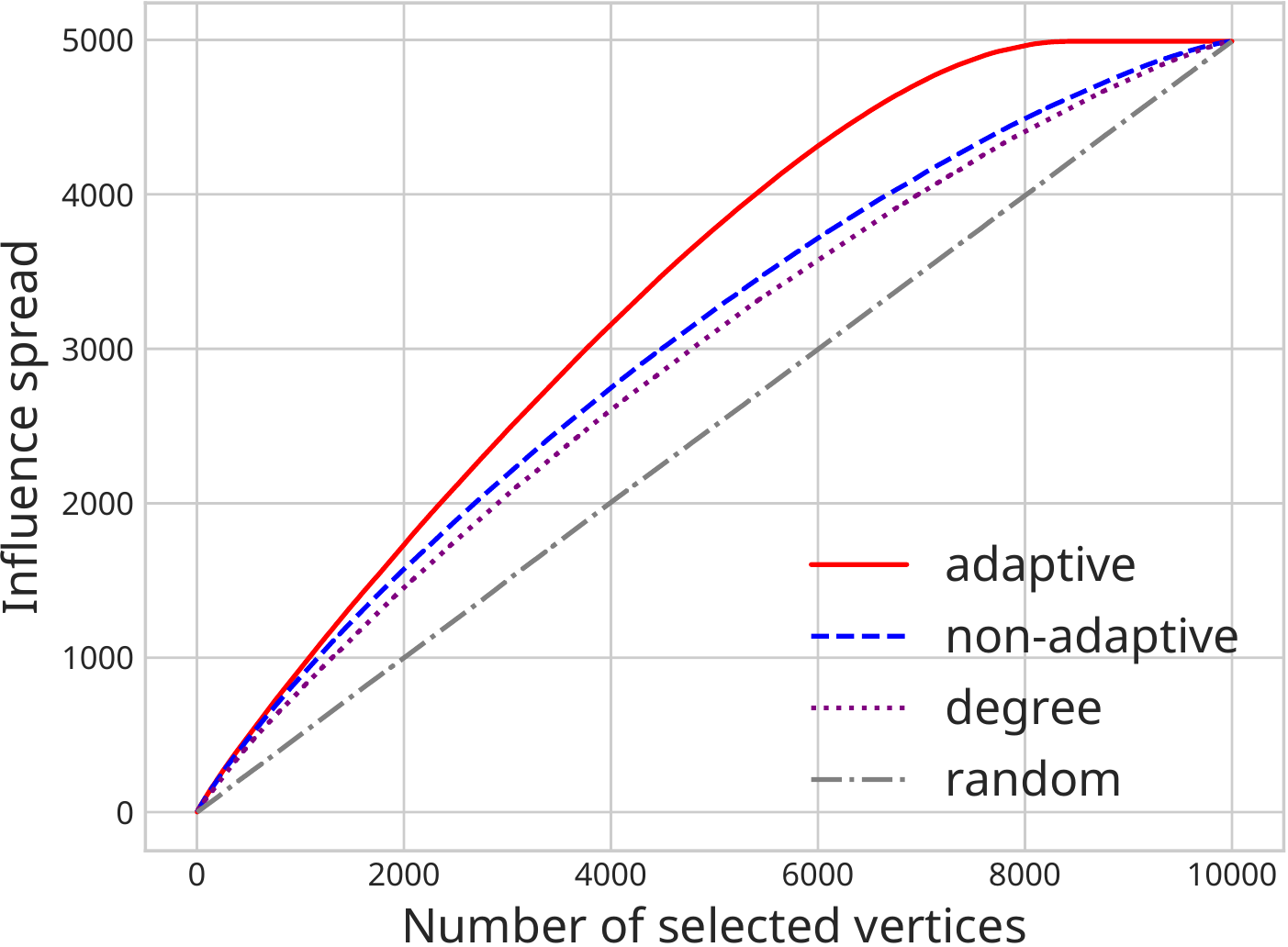}\label{fig:infmax_random_lt}
	}
	\subfigure[infmax, synth., extended linear thre.]{
		\includegraphics[width=0.29\textwidth]{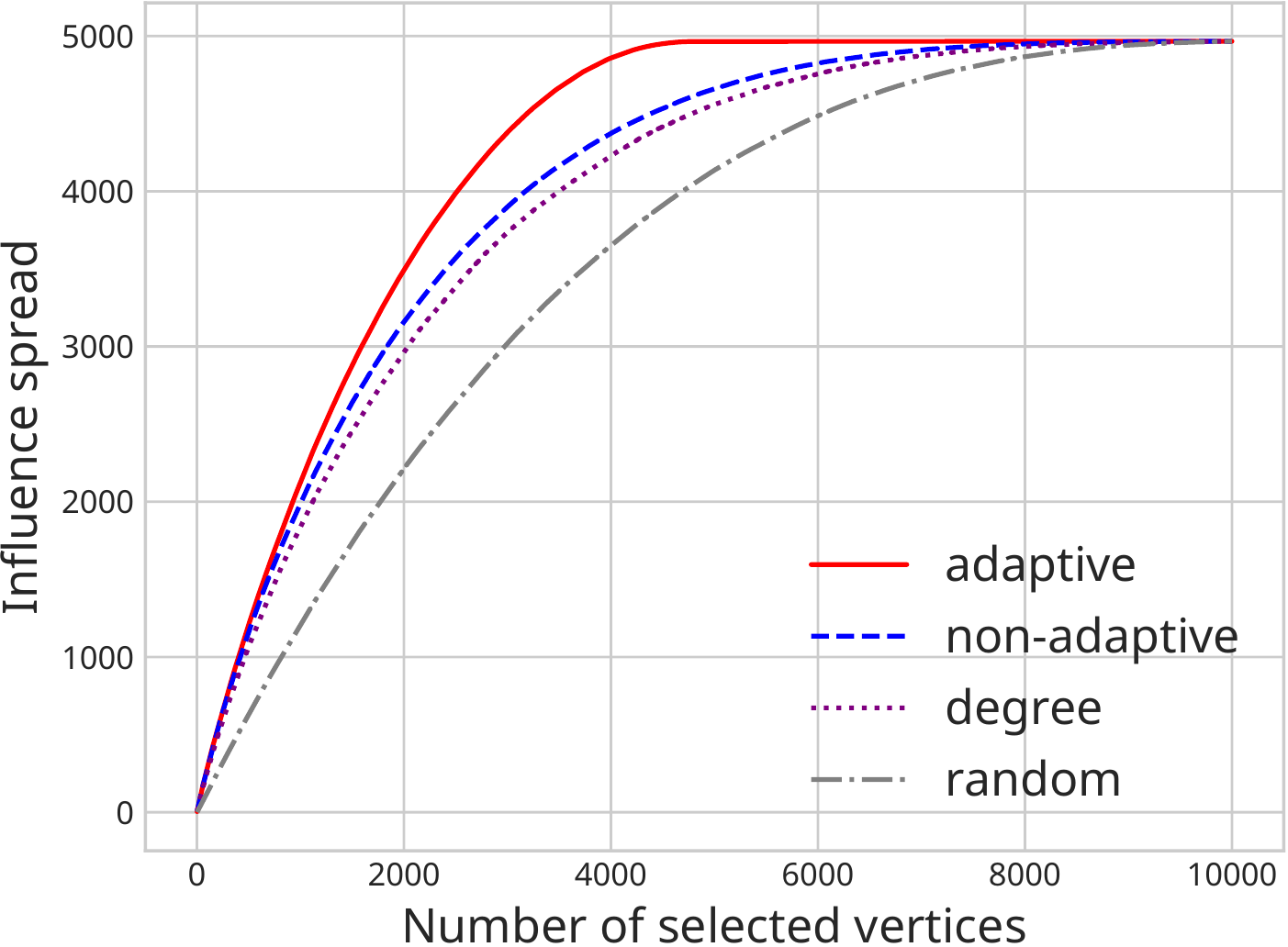}\label{fig:infmax_random_elt}
	}
	\subfigure[infmax, yahoo, linear threshold]{
		\includegraphics[width=0.29\textwidth]{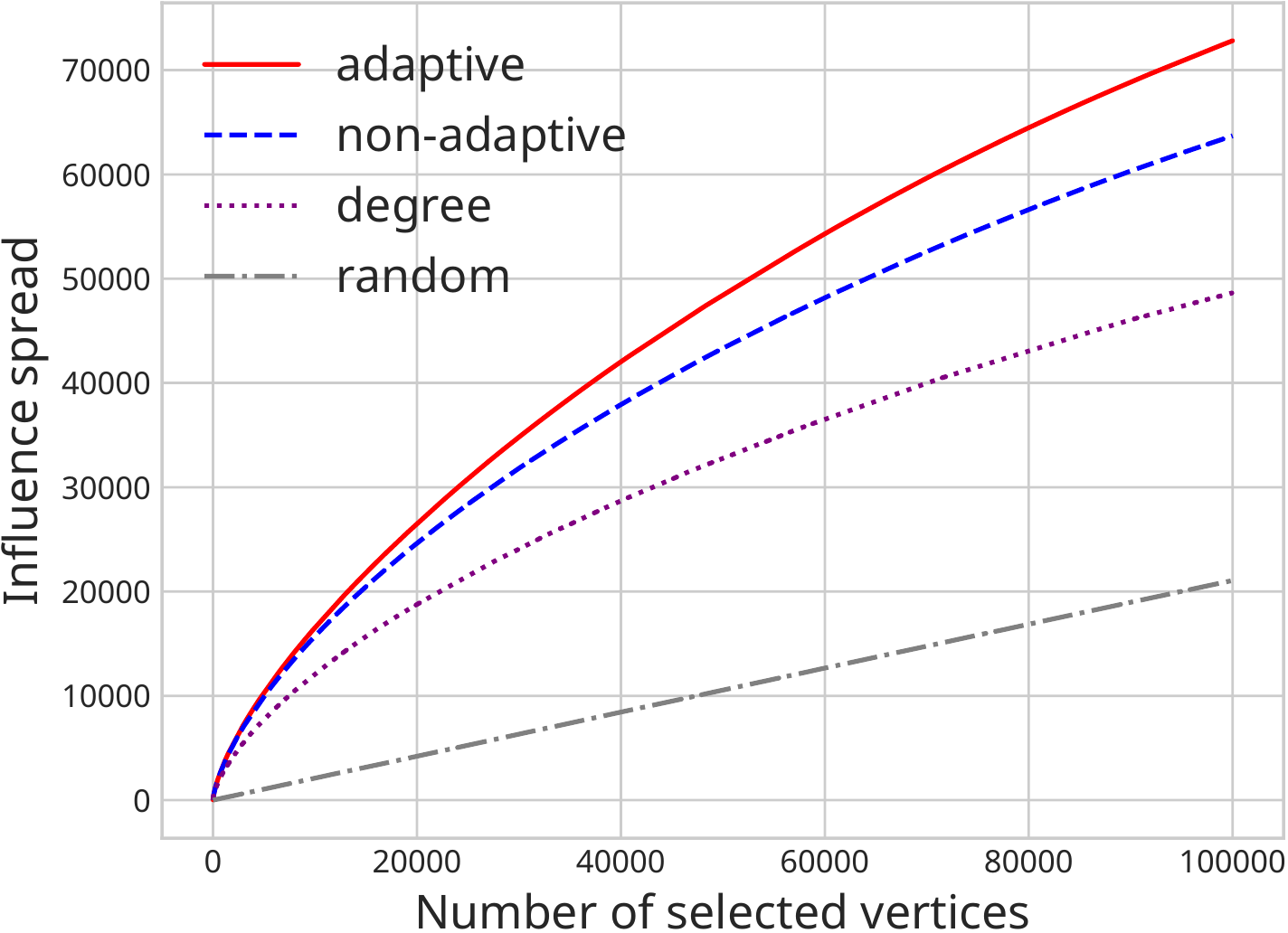}\label{fig:infmax_yahoo_lt}
	}
	\subfigure[infmax, yahoo, extended linear thre.]{
		\includegraphics[width=0.29\textwidth]{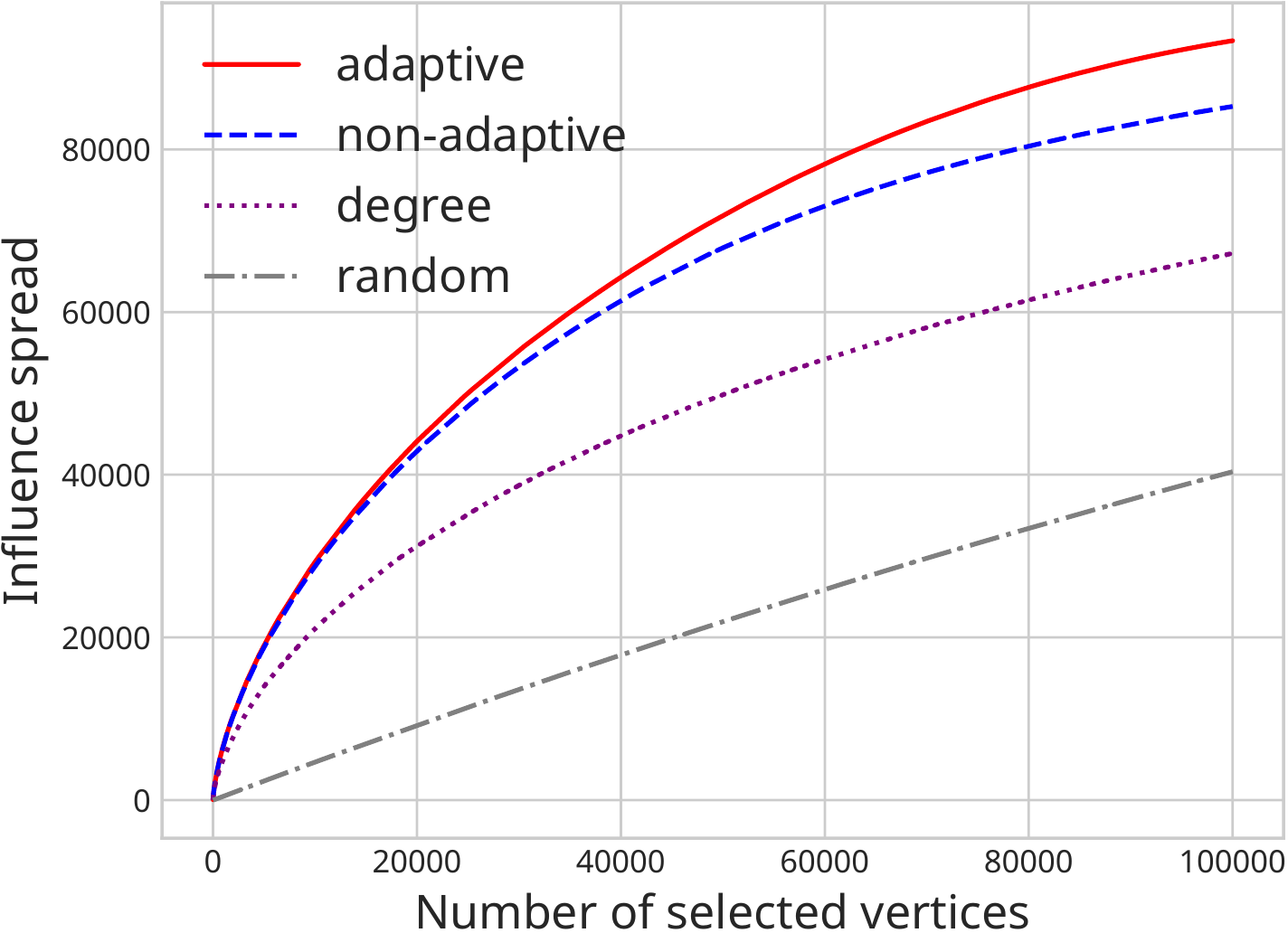}\label{fig:infmax_yahoo_elt}
	}
	\subfigure[feature, $\sigma=0.1$]{
		\includegraphics[width=0.29\textwidth]{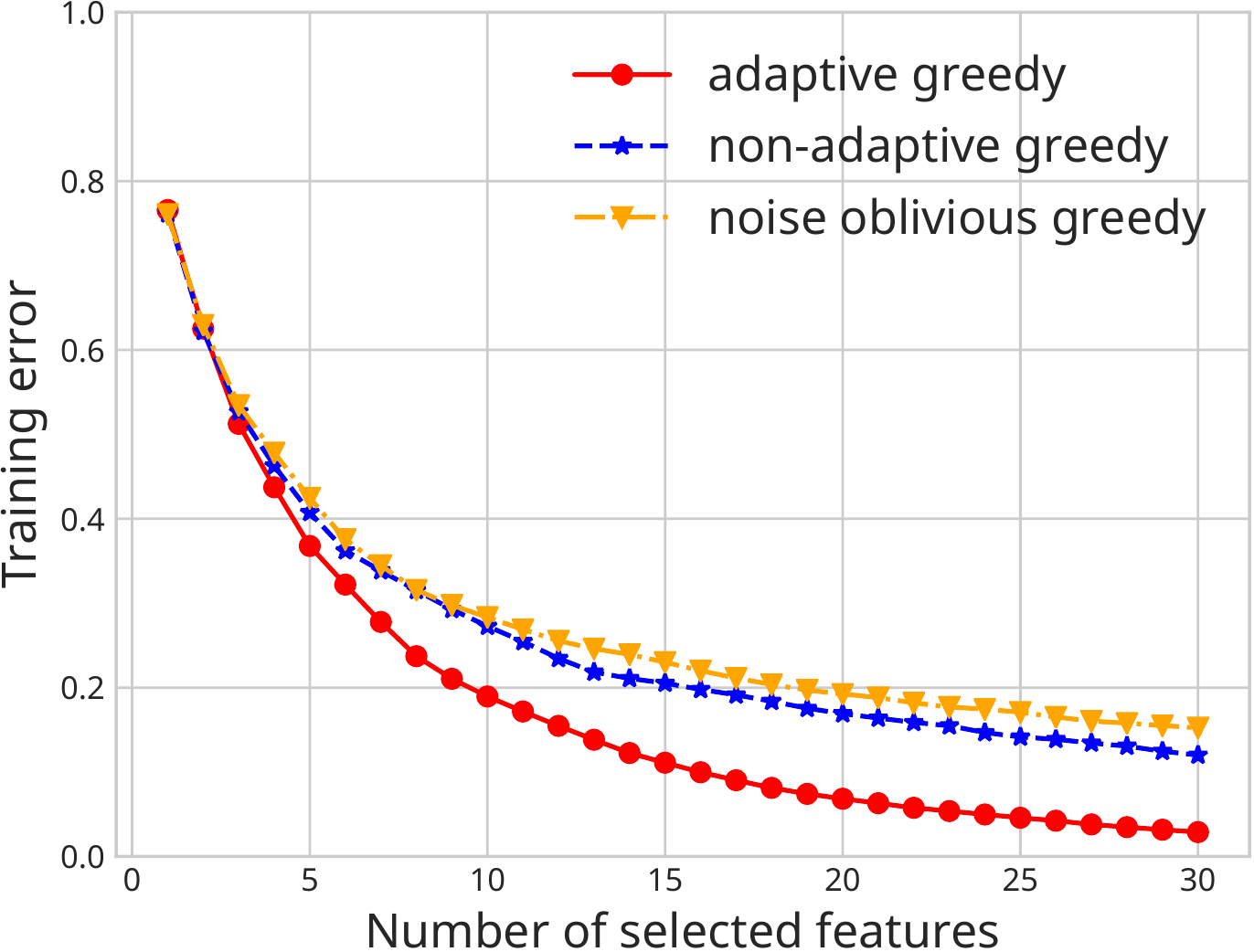}\label{fig:feature_noise01}
	}
	\subfigure[feature, $\sigma=0.2$]{
		\includegraphics[width=0.29\textwidth]{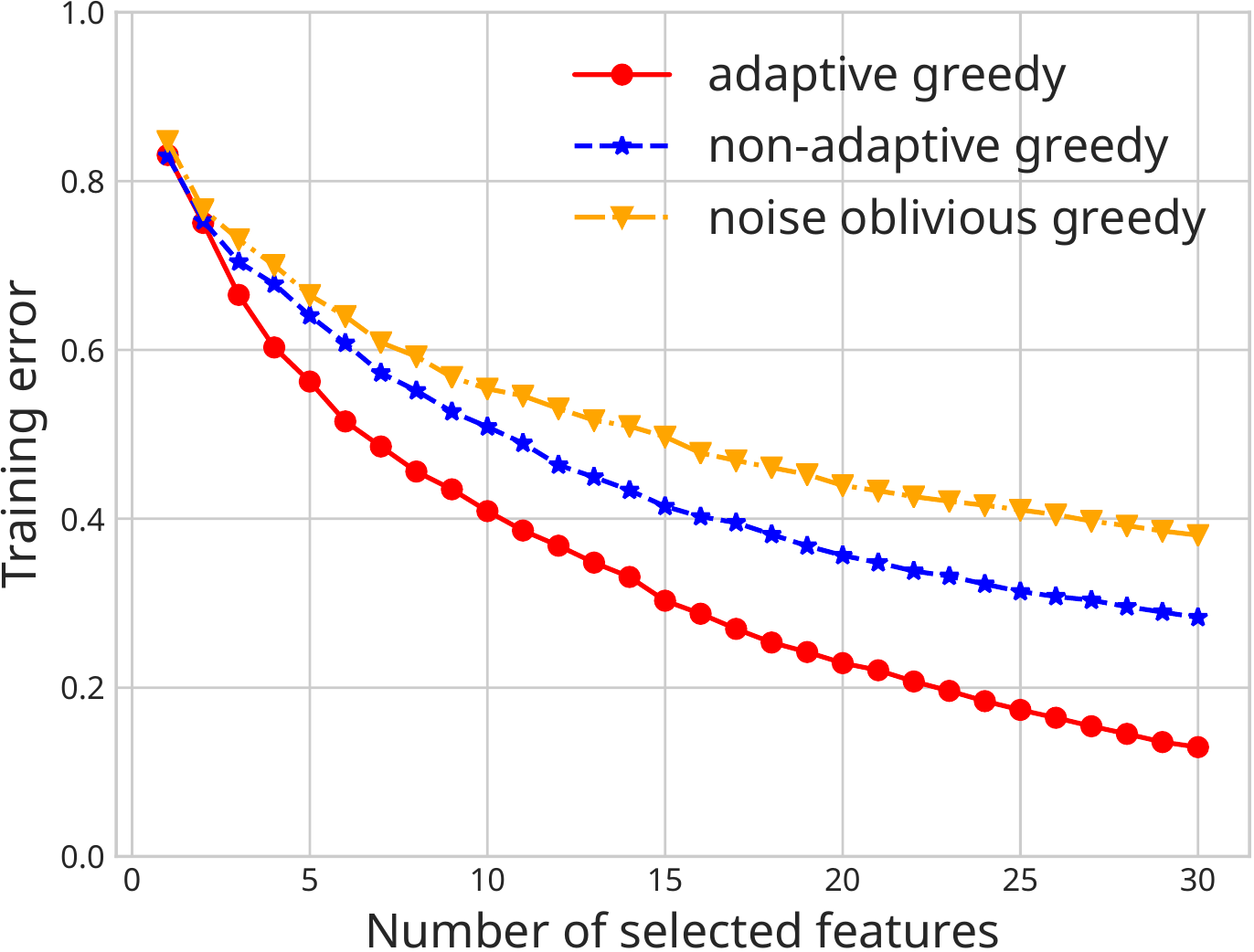}\label{fig:feature_noise02}
	}
	\caption{Experimental results on adaptive influence maximization \subref{fig:infmax_random_lt}--\subref{fig:infmax_yahoo_elt} and adaptive feature selection \subref{fig:feature_noise01}--\subref{fig:feature_noise02}. \subref{fig:infmax_random_lt} and \subref{fig:infmax_random_elt} are the results on synthetic datasets with the linear threshold model and extended linear threshold model, respectively. \subref{fig:infmax_yahoo_lt} and \subref{fig:infmax_yahoo_elt} are the results on Yahoo!\ dataset \cite{Yahoo} with the linear threshold model and extended linear threshold model, respectively. \subref{fig:feature_noise01} and \subref{fig:feature_noise02} are the results on synthetic datasets with uniform noise distribution on $[-\sigma, \sigma]$ with $\sigma = 0.1, 0.2$, respectively.}
\end{figure*}

\section{Experiments}\label{sec:experiment}
We conduct experiments on two applications: adaptive influence maximization and adaptive feature selection.
For each setting, we conduct $20$ trials and plot their mean values. 

\subsection{Adaptive Influence Maximization}
\paragraph{Datasets.}
We conduct experiments on two datasets of adaptive influence maximization. 
The first dataset is a synthetic bipartite graph generated randomly according to Erd\"{o}s--Renyi rule. 
We set the number of source and sink vertices to 10000, i.e., $|V| = |U| = 10000$.
For each pair $(v, u) \in V \times U$, we add an edge between $v$ and $u$ with probability $0.001$.
The second dataset is Yahoo! Search Marketing Advertiser--Phrase Bipartite Graph \cite{Yahoo}, which is a bipartite graph representing relationships between advertisers and search phrases; 
we have $|V| = 459678$, $|U| = 193582$, and $|A| = 2278448$.
For both datasets, the weight of each vertex in $U$ is drawn from the uniform distribution on $[0, 1]$.

\paragraph{Diffusion Model.}
We consider two diffusion models.
The first one is the linear threshold model.
The probability that each edge $(v, u) \in A$ is alive is set to the reciprocal of the degree of the sink vertex, that is, $1 / |\delta_{-}(v)|$.
As the second diffusion model, we consider an extended version of the linear threshold model, which is also a special case of the triggering model.
In this model, for each sink vertex $v$, 
the subset of incoming live edges is determined as follows. 
We sample $t$ edges with replacement from $\delta_{-}(v)$ uniformly at random, 
and an edge turns alive if it is sampled at least once. 
In our experiments, parameter $t$ is set to $3$.

\paragraph{Benchmarks.}
We compare the adaptive greedy algorithm with three non-adaptive benchmarks. The first benchmark is the non-adaptive greedy algorithm, called \textsf{non-adaptive}, which is a standard greedy algorithm \cite{NWF78} for maximizing the expected value of the objective function $\bbE_\Phi [f( \cdot , \Phi)]$. The second benchmark is \textsf{Degree}, which selects the set of vertices with the top-$k$ largest degree. The third benchmark is \textsf{Random}, which selects a random subset of size $k$.

\paragraph{Results.}
Objective values achieved by the algorithms are shown in \Cref{fig:infmax_random_lt,fig:infmax_random_elt,fig:infmax_yahoo_lt,fig:infmax_yahoo_elt}.
In all settings, the adaptive greedy algorithm outperforms all the benchmarks. 

\subsection{Adaptive Feature Selection}
\paragraph{Datasets.}
We use synthetic datasets generated randomly as follows.
First we determine the mean $\bbE_\Phi[\bfA(\Phi)] \in \bbR^{m \times n}$ according to the uniform distribution on $[0, 1]$.
After that, each column is normalized so that its mean is $0$ and its standard deviation is $1$.
We obtain $\bfA(\phi)$ by adding $\epsilon \in \bbR^{m \times n}$ to $\bbE_\Phi[\bfA(\Phi)]$, 
where each element of $\epsilon$ is 
drawn from the uniform distribution on $[- \sigma, \sigma]$.
We consider two settings: $\sigma = 0.1$ and $0.2$. 
We select a random sparse subset $S^*$ of features such that $|S^*| = 30$, 
and we let $\bfy = \bfA(\phi)_{S^*} \bfw$ 
be the response vector, where each element of $\bfw \in \bbR^{S}$ is drawn from the standard normal distribution.
In all settings, we set $n = 1000$ and $m = 100$.

\paragraph{Benchmarks.}
We compare the adaptive greedy algorithm with two benchmarks.
The first benchmark is the non-adaptive greedy algorithm.
Regarding the adaptive and non-adaptive greedy algorithms, it is hard to evaluate the exact values of the objective functions, and so we approximately evaluate them by sampling $\bfA(\Phi)$ randomly according to posterior distributions.
The second benchmark is the noise-oblivious greedy algorithm, a non-adaptive algorithm that greedily selects a subset based on the mean, $\bbE_\Phi [\bfA(\Phi)]$.

\paragraph{Results.}
The results are shown in \Cref{fig:feature_noise01,fig:feature_noise02}.
In both settings, the adaptive greedy algorithm outperforms the two benchmarks.

\section{Related Work}\label{sec:related}
\paragraph{Comparison with \citep{Kusner14}.}
To our knowledge, the first attempt to generalize submodularity ratio to the adaptive setting is \citep{Kusner14}.
They defined \textit{approximate adaptive submodularity}, 
a notion that is similar to ours, as follows: 
\begin{equation}
\gamma = \min_{S \subseteq V, \psi} \frac{\sum_{v \in S} \Delta(v | \psi)}{\Delta(S | \psi)}.
\end{equation}
The key difference is that they did not replace subset $S$ with policy $\pi$. 
In \Cref{sec:counter-kusner}, we show that the approximate adaptive submodularity is not sufficient for providing an approximation guarantee of the adaptive greedy algorithm.

\paragraph{Comparison with \citep{YGO17}.}
Another attempt to relax adaptive submodularity is presented in \cite{YGO17}. They introduced \textit{$\zeta$-weakly adaptive submodular functions} as follows:
\begin{defn}[$\zeta$-weak adaptive submodularity]
	Let $f \colon 2^V \times \calY^V \to \bbR$ be a set function and $p$ be a distribution of $\phi$.
	For any $\zeta \ge 1$, we say $f$ is adaptive submodular with respect to $p$ if for any partial realization $\psi \subseteq \psi'$ and any element $v \in V \setminus \dom(\psi')$, it holds
	$
	\zeta \Delta(v | \psi) \ge \Delta(v | \psi').
	$
	Let $\zeta^*$ be the infimum of $\zeta$ satisfying the above inequality.
\end{defn}
Analogous to our adaptive submodularity ratio, 
one can readily see that $1$-weak adaptive submodularity is equivalent to the adaptive submodularity. 
In general, however, there is a difference between the two notions; 
the adaptive submodularity ratio can be bounded from below by $1/\zeta^*$, implying that 
it is more demanding to bound the value of $\zeta^*$ 
than that of the adaptive submodularity ratio. 
\begin{prop}\label{prop:curvature-ratio-bound}
	For any set function $f \colon 2^V \times \calY^V \to \bbR$ and distribution $p$, we have
	$
	\frac{1}{\zeta^*} \le \min_{k \in \bbZ_{\ge 0}, \psi} \gamma_{\psi, k}.
	$
\end{prop}
We provide a proof in \Cref{subsec:proof-yong}.
\citet{YGO17} studied a problem called \textit{group-based active diagnosis} and gave a bound of $\zeta$, but some vital assumptions seem to have been missed.
In \Cref{sec:counter-yong}, we provide a problem instance in which their bound does not hold. 
We also present instances of adaptive influence maximization and adaptive feature selection for which our framework provides strictly better approximation ratios than those obtained with the weak adaptive submodularity in \Cref{subsec:comparison-yong-infmax,subsec:comparison-yong-feature}.  

\paragraph{Adaptive Submodularity.}
Adaptive submodularity was proposed by \citet{GK11}. 
There are several attempts to adaptively maximize set functions that do not satisfy adaptive submodularity (e.g., \cite{Kusner14,YGO17}). 
\citet{Chen15} analyzed the greedy policy focusing on the maximization of mutual information, which does not have adaptive submodularity. 

\paragraph{Submodularity Ratio.}
Submodularity ratio was proposed by \citet{Das2011} for sparse regression with squared $\ell_2$ loss. Recently, \citet{Elenberg18} extended this result to more general loss functions with restricted strong convexity and restricted smoothness. \citet{Bogunovic18} proposed the notion of \textit{supermodularity ratio}. \citet{Bian17} provided a guarantee of the non-adaptive greedy algorithm 
for the case where the total curvature and submodularity ratio of objective functions are bounded.

\paragraph{Influence Maximization.}
Influence maximization was proposed by \citet{KKT03}. An adaptive version of influence maximization was first considered by \citet{GK11}. They showed that this objective function satisfies adaptive submodularity under the independent cascade model in general graphs. Influence maximization on a bipartite graph has been studied for applications to advertisement selection \cite{Alon12,Soma14}. This problem setting was extended to the adaptive setting by \citet{Hatano16}, but only the independent cascade model was considered. 
The curvature of its objective function was studied by \citet{Maehara17}.

\paragraph{Feature Selection.}
\citet{KKLP17} considered the problem called adaptive feature selection, but their problem setting is different from ours.
In their setting, the learner solves feature selection problems multiple times.
They studied the adaptivity among the multiple rounds, 
while we studied the adaptivity inside of a single round. 

\section*{Acknowledgements}
K.F. was supported by JSPS KAKENHI Grant Number JP 18J12405.

\nocite{langley00}

\bibliography{main}

\begin{thebibliography}{24}
\providecommand{\natexlab}[1]{#1}
\providecommand{\url}[1]{\texttt{#1}}
\expandafter\ifx\csname urlstyle\endcsname\relax
  \providecommand{\doi}[1]{doi: #1}\else
  \providecommand{\doi}{doi: \begingroup \urlstyle{rm}\Url}\fi

\bibitem[Yah()]{Yahoo}
Yahoo! webscope dataset: G1 - {Y}ahoo! {S}earch {M}arketing
  {A}dvertiser-{P}hrase {B}ipartite {G}raph, {V}ersion 1.0.
\newblock URL \url{https://webscope.sandbox.yahoo.com/}.

\bibitem[Alon et~al.(2012)Alon, Gamzu, and Tennenholtz]{Alon12}
Alon, N., Gamzu, I., and Tennenholtz, M.
\newblock Optimizing budget allocation among channels and influencers.
\newblock In \emph{Proceedings of the 21st World Wide Web Conference 2012,
  {WWW} 2012}, pp.\  381--388, 2012.

\bibitem[Balkanski \& Singer(2018)Balkanski and Singer]{BS18}
Balkanski, E. and Singer, Y.
\newblock The adaptive complexity of maximizing a submodular function.
\newblock In \emph{Proceedings of the 50th Annual {ACM} {SIGACT} Symposium on
  Theory of Computing, {STOC} 2018}, pp.\  1138--1151, 2018.

\bibitem[Bian et~al.(2017)Bian, Buhmann, Krause, and Tschiatschek]{Bian17}
Bian, A.~A., Buhmann, J.~M., Krause, A., and Tschiatschek, S.
\newblock Guarantees for greedy maximization of non-submodular functions with
  applications.
\newblock In \emph{Proceedings of the 34th International Conference on Machine
  Learning, {ICML} 2017}, pp.\  498--507, 2017.

\bibitem[Bogunovic et~al.(2018)Bogunovic, Zhao, and Cevher]{Bogunovic18}
Bogunovic, I., Zhao, J., and Cevher, V.
\newblock Robust maximization of non-submodular objectives.
\newblock In \emph{Proceedings of the 21st International Conference on
  Artificial Intelligence and Statistics, {AISTATS} 2018}, pp.\  890--899,
  2018.

\bibitem[Chen et~al.(2015)Chen, Hassani, Karbasi, and Krause]{Chen15}
Chen, Y., Hassani, S.~H., Karbasi, A., and Krause, A.
\newblock Sequential information maximization: When is greedy near-optimal?
\newblock In \emph{Proceedings of The 28th Conference on Learning Theory,
  {COLT} 2015}, pp.\  338--363, 2015.

\bibitem[Das \& Kempe(2011)Das and Kempe]{Das2011}
Das, A. and Kempe, D.
\newblock Submodular meets spectral: Greedy algorithms for subset selection,
  sparse approximation and dictionary selection.
\newblock In \emph{Proceedings of the 28th International Conference on Machine
  Learning, {ICML} 2011}, pp.\  1057--1064, 2011.

\bibitem[Elenberg et~al.(2017)Elenberg, Dimakis, Feldman, and Karbasi]{EDFK17}
Elenberg, E.~R., Dimakis, A.~G., Feldman, M., and Karbasi, A.
\newblock Streaming weak submodularity: Interpreting neural networks on the
  fly.
\newblock In \emph{Advances in Neural Information Processing Systems 30}, pp.\
  4047--4057, 2017.

\bibitem[Elenberg et~al.(2018)Elenberg, Khanna, Dimakis, and
  Negahban]{Elenberg18}
Elenberg, E.~R., Khanna, R., Dimakis, A.~G., and Negahban, S.
\newblock Restricted strong convexity implies weak submodularity.
\newblock \emph{Ann. Statist.}, 46\penalty0 (6B):\penalty0 3539--3568, 2018.

\bibitem[Gabillon et~al.(2013)Gabillon, Kveton, Wen, Eriksson, and
  Muthukrishnan]{GKWEM13}
Gabillon, V., Kveton, B., Wen, Z., Eriksson, B., and Muthukrishnan, S.
\newblock Adaptive submodular maximization in bandit setting.
\newblock In \emph{Advances in Neural Information Processing Systems 26}, pp.\
  2697--2705, 2013.

\bibitem[Golovin \& Krause(2011)Golovin and Krause]{GK11}
Golovin, D. and Krause, A.
\newblock Adaptive submodularity: Theory and applications in active learning
  and stochastic optimization.
\newblock \emph{J. Artif. Intell. Res.}, 42:\penalty0 427--486, 2011.

\bibitem[Golovin et~al.(2010)Golovin, Krause, and Ray]{GKR10}
Golovin, D., Krause, A., and Ray, D.
\newblock Near-optimal bayesian active learning with noisy observations.
\newblock In \emph{Advances in Neural Information Processing Systems 23}, pp.\
  766--774, 2010.

\bibitem[Hatano et~al.(2016)Hatano, Fukunaga, and Kawarabayashi]{Hatano16}
Hatano, D., Fukunaga, T., and Kawarabayashi, K.
\newblock Adaptive budget allocation for maximizing influence of
  advertisements.
\newblock In \emph{Proceedings of the 25th International Joint Conference on
  Artificial Intelligence, {IJCAI} 2016}, pp.\  3600--3608, 2016.

\bibitem[Javdani et~al.(2014)Javdani, Chen, Karbasi, Krause, Bagnell, and
  Srinivasa]{JCKKBS14}
Javdani, S., Chen, Y., Karbasi, A., Krause, A., Bagnell, D., and Srinivasa,
  S.~S.
\newblock Near optimal bayesian active learning for decision making.
\newblock In \emph{Proceedings of the 17th International Conference on
  Artificial Intelligence and Statistics, {AISTATS} 2014}, pp.\  430--438,
  2014.

\bibitem[Kale et~al.(2017)Kale, Karnin, Liang, and P{\'a}l]{KKLP17}
Kale, S., Karnin, Z., Liang, T., and P{\'a}l, D.
\newblock Adaptive feature selection: Computationally efficient online sparse
  linear regression under {RIP}.
\newblock In \emph{Proceedings of the 34th International Conference on Machine
  Learning, {ICML} 2017}, pp.\  1780--1788, 2017.

\bibitem[Kempe et~al.(2003)Kempe, Kleinberg, and Tardos]{KKT03}
Kempe, D., Kleinberg, J.~M., and Tardos, {\'{E}}.
\newblock Maximizing the spread of influence through a social network.
\newblock In \emph{Proceedings of the 9th {ACM} {SIGKDD} International
  Conference on Knowledge Discovery and Data Mining, {KDD} 2003}, pp.\
  137--146, 2003.

\bibitem[Khanna et~al.(2017)Khanna, Elenberg, Dimakis, Negahban, and
  Ghosh]{KEDNG17}
Khanna, R., Elenberg, E.~R., Dimakis, A.~G., Negahban, S., and Ghosh, J.
\newblock Scalable greedy feature selection via weak submodularity.
\newblock In \emph{Proceedings of the 20th International Conference on
  Artificial Intelligence and Statistics, {AISTATS} 2017}, pp.\  1560--1568,
  2017.

\bibitem[Kusner(2014)]{Kusner14}
Kusner, M.~J.
\newblock Approximately adaptive submodular maximization.
\newblock In \emph{NIPS Workshop on Discrete and Combinatorial Problems in
  Machine Learning}, 2014.

\bibitem[Leskovec et~al.(2007)Leskovec, Krause, Guestrin, Faloutsos,
  VanBriesen, and Glance]{LKGFVG07}
Leskovec, J., Krause, A., Guestrin, C., Faloutsos, C., VanBriesen, J.~M., and
  Glance, N.~S.
\newblock Cost-effective outbreak detection in networks.
\newblock In \emph{Proceedings of the 13th {ACM} {SIGKDD} International
  Conference on Knowledge Discovery and Data Mining, {KDD} 2007}, pp.\
  420--429, 2007.

\bibitem[Maehara et~al.(2017)Maehara, Kawase, Sumita, Tono, and
  Kawarabayashi]{Maehara17}
Maehara, T., Kawase, Y., Sumita, H., Tono, K., and Kawarabayashi, K.
\newblock Optimal pricing for submodular valuations with bounded curvature.
\newblock In \emph{Proceedings of the 31st {AAAI} Conference on Artificial
  Intelligence, {AAAI} 2017}, pp.\  622--628, 2017.

\bibitem[Nemhauser et~al.(1978)Nemhauser, Wolsey, and Fisher]{NWF78}
Nemhauser, G.~L., Wolsey, L.~A., and Fisher, M.~L.
\newblock An analysis of approximations for maximizing submodular set functions
  - {I}.
\newblock \emph{Math. Program.}, 14\penalty0 (1):\penalty0 265--294, 1978.

\bibitem[Soma et~al.(2014)Soma, Kakimura, Inaba, and Kawarabayashi]{Soma14}
Soma, T., Kakimura, N., Inaba, K., and Kawarabayashi, K.
\newblock Optimal budget allocation: Theoretical guarantee and efficient
  algorithm.
\newblock In \emph{Proceedings of the 31th International Conference on Machine
  Learning, {ICML} 2014}, pp.\  351--359, 2014.

\bibitem[Tang et~al.(2014)Tang, Xiao, and Shi]{TXS14}
Tang, Y., Xiao, X., and Shi, Y.
\newblock Influence maximization: near-optimal time complexity meets practical
  efficiency.
\newblock In \emph{Proceedings of the 2014 {ACM} {SIGMOD} International
  Conference on Management of Data, {SIGMOD} 2014}, pp.\  75--86, 2014.

\bibitem[Yong et~al.(2017)Yong, Gao, and Ozay]{YGO17}
Yong, S.~Z., Gao, L., and Ozay, N.
\newblock Weak adaptive submodularity and group-based active diagnosis with
  applications to state estimation with persistent sensor faults.
\newblock In \emph{2017 American Control Conference (ACC)}, pp.\  2574--2581,
  2017.

\end{thebibliography}
\bibliographystyle{icml2019}


\appendix
\setcounter{equation}{0}
\setcounter{prop}{0}
\setcounter{lem}{0}
\setcounter{algorithm}{0}
\renewcommand{\theequation}{A\arabic{equation}}




\clearpage
\onecolumn
\begin{center}
	{\fontsize{18pt}{0pt}\selectfont \bf Appendices}
\end{center}

\appendix

\section{Proof for Adaptive Submodularity Ratio}\label{sec:app-ratio}
\begin{proof}[Proof of \Cref{prop:ratio_one}]
	First we deal with the ``if'' part.
	Let $\psi_v$ be the partial realization just before $v$ is selected in $\pi$.
	If there are multiple partial realizations $\psi$ such that $\pi(\psi) = v$, we can duplicate $v$ and take them to be different elements.
	From adaptive submodularity, for any partial realization $\psi$ and policy $\pi$, we have
	\begin{align}
	\Delta(\pi | \psi) 
	&= \sum_{v \in V} \Pr(v \in E(\pi, \Phi) | \Phi \sim \psi) \Delta(v | \psi \cup \psi_v) \\
	&\le \sum_{v \in V} \Pr(v \in E(\pi, \Phi) | \Phi \sim \psi) \Delta(v | \psi).
	\end{align}
	Thus we can see $\gamma_{\psi, k} \ge 1$. 
	Moreover, 
	if $\pi$ is a policy that selects a single element, 
	the above inequality holds with equality.
	These two facts imply $\gamma_{\psi, k} = 1$.
	
	Next we deal with the ``only if'' part.
	Let $\psi \subseteq \psi'$ be any partial realization such that $|\psi| + 1 = |\psi'|$ and $v \in V \setminus \dom(\psi')$ be any element.
	We define $u \in \dom(\psi') \setminus \dom(\psi)$ to be the additional element and $y$ its state in $\psi'$, i.e., $\psi' = \psi \cup \{(u, y)\}$.
	Let us consider a policy $\pi$ that first selects $u$ and, 
	if $\phi(u) = y$, proceeds to select $v$.
	From the assumption, we have $\gamma_{\psi, 2} = 1$, and thus  $\Delta(\pi|\psi) \le \sum_{v \in V} \Pr(v \in E(\pi, \Phi)) \Delta(v | \psi)$.
	We can calculate the left and right hand sides as follows:
	\begin{align}
		\text{(LHS)} &= \Delta(u|\psi) + \Pr(\Phi(u) = y | \Phi \sim \psi) \Delta(v | \psi'),\\
		\text{(RHS)} &= \Delta(u|\psi) + \Pr(\Phi(u) = y | \Phi \sim \psi) \Delta(v | \psi).
	\end{align}
	Therefore, we obtain $\Delta(v | \psi') \le \Delta(v | \psi)$. By sequentially concatenating inequalities of this type, we can show that the statement holds for any $\psi \subseteq \psi'$.
\end{proof}

\section{Proof for the Adaptive Greedy Algorithm}\label{sec:app-greedy}

To prove \Cref{thm:adaptive-greedy}, we introduce a lemma provided by \citet{GK11}.
Let $\pi' @ \pi$ be a concatenated policy, i.e., a policy that executes $\pi'$ as if from scratch after executing $\pi$.
Adaptive monotonicity is known to be equivalent to the following condition: 
\begin{lem}[{Adopted from \citep[Lemma A.8]{GK11}}]\label{lem:monotone}
	Fix $f:2^V \times \calY^V \to \mathbb{R}_{\ge 0}$.
	Then we have $\Delta(v|\psi) \ge 0$ for all $\psi$ with $p(\psi) > 0$ and all $v \in V$ if and only if for all policies $\pi$ and $\pi'$, we have $\favg(\pi) \le \favg(\pi' @ \pi)$.
\end{lem}

\begin{proof}[Proof of \Cref{thm:adaptive-greedy}]
	Let $\psi$ be any possible partial realization that can appear while executing the adaptive greedy policy $\pi$.
	Since $\pi$ stops after $\ell$ steps, we have $|\psi| \le \ell$.
	According to the definition of adaptive submodularity ratio, we have
	\begin{equation}\label{eq:ratio-dependent-bound}
	\gamma_{\ell, k} \Delta(\pi^* | \psi) \le \sum_{v \in V} \Pr(v \in E(\pi^*, \Phi) | \Phi \sim \psi) \Delta(v | \psi) \le k \max_{v \in V} \Delta(v | \psi)
	\end{equation}
	since $\sum_{v \in V} \Pr(v \in E(\pi^*, \Phi) | \Phi \sim \psi) = \bbE[|E(\pi^*, \Phi)|] \le k$.
	Let $\Psi$ be a random partial realization observed by executing $\pi_{[i]}$, where $\pi_{[i]}$ is a policy obtained by running $\pi$ until it terminates or it selects $i$ elements.
	Formally, $\Psi$ conforms to the distribution $p_\Psi(\psi) \coloneqq \Pr(\Psi = \psi \mid \exists \phi, ~ \psi = \{(v, \phi(v)) \mid v \in E(\pi_{[i]}, \phi)\})$.
	Then we can lower-bound the expected single step gain as follows:
	\begin{align}
	\favg(\pi_{[i+1]}) - \favg(\pi_{[i]})
	&= \bbE \left[ \max_{v \in V} \Delta(v | \Psi) \right] \tag{due to the property of the adaptive greedy algorithm}\\
	&\ge \bbE \left[ \frac{\gamma_{\ell, k}}{k} \Delta(\pi^* | \Psi) \right] \tag{due to \eqref{eq:ratio-dependent-bound}}\\
	&= \frac{\gamma_{\ell, k}}{k} \left( \favg (\pi_{[i]} @ \pi^* ) - \favg(\pi_{[i]}) \right)\\
	&\ge \frac{\gamma_{\ell, k}}{k} \left( \favg (\pi^* ) - \favg(\pi_{[i]}) \right). \tag{due to adaptive monotonicity and \Cref{lem:monotone}}
	\end{align}
	
	Let $\Delta_i \coloneqq \favg(\pi^*) - \favg(\pi_{[i]})$.
	The above inequality can be rewritten as $\Delta_i - \Delta_{i+1} \ge \gamma_{\ell, k} \Delta_i / k$, which implies $\Delta_{i+1} \le (1 - \gamma_{\ell, k} / k) \Delta_i$.
	By repeatedly using this inequality, we obtain $\Delta_\ell \le (1 - \gamma_{\ell, k} / k)^\ell \Delta_0 \le \exp(- \gamma_{\ell, k} \ell / k) \favg(\pi^*)$.
	Consequently, we have $\favg(\pi) \ge (1 - \exp(- \gamma_{\ell, k} \ell / k)) \favg(\pi^*)$.
\end{proof}

\section{Proofs for Adaptivity Gaps}\label{sec:app-gap}
\begin{proof}[Proof of \Cref{thm:gap}]
	Let $\pi^*_\non$ be an optimal non-adaptive policy and $\pi^*$ be an optimal adaptive policy.
	Since $\pi^*_\non$ is a non-adaptive policy, it selects the same subset for all $\phi$, i.e., $E(\pi^*_\non, \phi) = E(\pi^*_\non, \phi')$ for all $\phi$ and $\phi'$. 
	Let $M \in \argmax \sum_{v \in M \colon |M| \le k} \Delta(v|\emptyset)$ and $\pi^M_\non$ the non-adaptive policy that selects $M$.
	From the optimality of $\pi^*_\non$, we have
	\begin{equation}
		\favg(\pi^*_\non) \ge \favg(\pi^M_\non).
	\end{equation}
	By the definition of the supermodularity ratio, we have
	\begin{equation}
		\Delta(\pi^M_\non | \emptyset) \ge \beta_{\emptyset, k} \sum_{v \in M} \Delta(v | \emptyset).
	\end{equation}
	Note that $\sum_{v \in V} \Pr(v \in E(\pi^*, \Phi)) \le k$ and $\Pr(v \in E(\pi^*, \Phi)) \le 1$ for each $v \in V$.
	Due to the definition of $M$, we have
	\begin{equation}
	\sum_{v \in M} \Delta(v | \emptyset) \ge \sum_{v \in V} \Pr(v \in E(\pi^*, \Phi)) \Delta(v | \emptyset).
	\end{equation}
	From the definition of adaptive submodularity ratio, we have
	\begin{equation}
	\sum_{v \in V} \Pr(v \in E(\pi^*, \Phi)) \Delta(v | \emptyset) \ge \gamma_{\emptyset, k} \Delta (\pi^*| \emptyset).
	\end{equation}
	Combining these inequalities, we have
	\begin{align}
		\favg(\pi^*_\non)
		&\ge \bbE_\Phi[f(\emptyset,\Phi)] + \Delta(\pi^M_\non | \emptyset)\\
		&\ge \beta_{\emptyset, k} \gamma_{\emptyset, k} (\bbE_\Phi[f(\emptyset,\Phi)] + \Delta (\pi^*| \emptyset))\\
		&= \beta_{\emptyset, k} \gamma_{\emptyset, k} \favg(\pi^*).
	\end{align}
\end{proof}

\begin{proof}[Proof of \Cref{cor:nonadaptive-approx}]
	From the approximation ratio, we have
	\begin{equation}
	\favg(\pi_\non) \ge \alpha \favg(\pi^*_\non).
	\end{equation}
	From \Cref{thm:gap}, we have
	\begin{equation}
	\favg(\pi^*_\non) \ge \beta_{\emptyset, k} \gamma_{\emptyset, k} \favg(\pi^*). 
	\end{equation}
	The above two inequalities imply the statement.
\end{proof}

From the following example, we can see that \Cref{thm:gap} is tight, i.e., for any rationals $\beta$ and $\gamma$ in $(0,1]$, there exist $f$ and $p$ such that the equality holds.

\begin{exmp}
Let $V = \{u\} \cup \bigcup_{i = 1}^M V_i$ be the ground set, where $V_i = \{v_i^1,\cdots,v_i^k\}$.
Let $V_0 = \emptyset$.
Let $\calY = \{0, 1, \cdots, M\}$.
We define the probability distribution $p$ such that $\phi(u) = y \in \calY$ with probability $p \in [0, 1/M]$ for each $y \neq 0$ and $\phi(u) = 0$ with probability $1 - pM$.
Other elements always in state $0$, i.e., $\phi(v) = 0$ with probability $1$ for all $v \in V \setminus \{u\}$.
We define the objective function $f$ as
\begin{equation}
	f(S, \phi) =
	\begin{cases}
		1 + a |S \cap V_{\phi(u)}| & (u \in S)\\
		1 + ap (|S| - 1) & (u \not\in S ~ \text{and} ~ |S| \ge 1)\\
		0 & (S = \emptyset),
	\end{cases}
\end{equation}
where $a \in \bbR_{\ge 0}$ is the parameter specified later.
We have $\Delta(v|\emptyset) = 1$ for all $v \in V$.
The supermodularity ratio $\beta_{\emptyset,k}$ of $\bbE[f(\cdot,\Phi)]$ is
\begin{equation}
	\beta_{\emptyset,k} = \frac{1 + (k-1) ap}{k}.
\end{equation}
The adaptive submodularity ratio $\gamma_{\emptyset,k}$ is
\begin{equation}
	\gamma_{\emptyset,k} = \frac{k}{1 + (k-1) apM}.
\end{equation}
The adaptivity gap is
\begin{equation}
	\gap_k(f,p) = \frac{1 + (k-1) ap}{1 + (k-1) apM}.
\end{equation}
For any rationals $\beta \in (0, 1]$ and $\gamma \in (0,1]$, there exist some $k, a, M$ such that $\gamma_{\emptyset,k} = \gamma$ and $\beta_{\emptyset,k} = \beta$.
\end{exmp}

\section{Proof for Adaptive Influence Maximization}\label{sec:app-infmax}
In this section, we provide the full proof for \Cref{thm:bipartite_ratio}. For the readability, we first give a proof for the case of the linear threshold model, which is a special case of the triggering model. After that, we give a proof for the case of the triggering model. 
\subsection{Proof for the Linear Threshold Model}\label{sec:app-infmax-linear}
\begin{proof}[Proof of \Cref{thm:bipartite_ratio} in the case of the linear threshold model]
	Let $V$ be the source vertices, $U$ the sink vertices, 
	and $A \subseteq V \times U$ the directed edges.
	For notational simplicity, assume that $G=(V \cup U, A)$ is a complete bipartite graph, i.e., $A = V \times U$.
	By setting $b(a) = 0$ for all edges $a\in A$ 
	that originally do not exist, we can assume this without loss of generality.
	Fix any $\psi' \subseteq \psi$ and $\pi \in \Pi_k$.
	It suffices to prove
	\begin{equation}\label{eq:influence-goal-inequality}
	\Delta(\pi | \psi') \le \frac{2k}{k+1} \sum_{v \in V} \Pr(v \in E(\pi, \Phi) | \Phi \sim \psi') \Delta(v | \psi').
	\end{equation}
	Let $\Delta_u(\cdot | \psi')$ be the expected marginal gain obtained by activating $u \in U$. 
	Below we explain that the above inequality can be separated for each $u \in U$; i.e., it is enough to prove the above inequality for the case where $w(u)>0$ for just one vertex $u \in U$ and $0$ for the others.  
	The objective function is the linear sum of the one for each $u \in U$: $\Delta(\cdot | \psi') = \sum_{u \in U} \Delta_u(\cdot | \psi')$.
	Therefore, the above inequality is decomposed into the sum of  
	\begin{equation}\label{eq:influence-v-inequality}
	\Delta_u(\pi | \psi') \le \frac{2k}{k+1} \sum_{v \in V} \Pr(v \in E(\pi, \Phi) | \Phi \sim \psi') \Delta_u(v | \psi')
	\end{equation} 
	for each $u \in U$.
	Note that the states of any $(v,u)\in A$ and $(v^\prime, u^\prime)\in A$ are independent of each other if $u \neq u^\prime$.
	Since the feedback about any $u' \in U$ such that $u' \neq u$ is never correlated with the states of edges pointing to $u$, we can regard the feedback about $u'$ as an independent random factor when considering \eqref{eq:influence-v-inequality}. Thus we can see that it is sufficient to consider the case of one sink vertex.
	Note that a randomized policy can be expressed as a linear sum of deterministic policies.
	Therefore, it is enough to consider the case where $\pi$ is a deterministic policy.
	Below we fix $u\in U$ and use $\Delta$ instead of $\Delta_u$ for notational ease. 
	We can assume $w(u) = 1$ without loss of generality.
	If $u$ has been already activated in $\psi'$, both sides of \eqref{eq:influence-v-inequality} are equal to zero; thus it holds trivially. We then consider the case where $u$ is not activated in $\psi'$.

	Let $\psi_v$ be the partial realization just before $v$ is selected in $\pi$.
	If there are multiple partial realizations $\psi$ such that $\pi(\psi) = v$, we can duplicate $v$ and consider them to be different elements.
	We can decompose $\Delta(\pi | \psi')$ as 
	\begin{equation}
	\Delta(\pi | \psi') = \sum_{v \in V} \Pr(v \in E(\pi, \Phi) | \Phi \sim \psi') \Delta(v | \psi' \cup \psi_v).
	\end{equation}
	The inequality that we aim to prove can be written as
	\begin{equation}\label{eq:influence-recursive-inequality-path}
		\sum_{v \in V} \Pr(v \in E(\pi, \Phi) | \Phi \sim \psi') \left\{ \frac{2k}{k+1} \Delta(v | \psi') - \Delta(v | \psi' \cup \psi_v) \right\} \ge 0.
	\end{equation}
	Since $\pi$ is a deterministic policy that observes only states of edges pointing to $u$, 
	there exists a path in policy tree $\pi$ wherein $u$ remains inactive; 
	in \Cref{fig:compensation} such a path is colored in thin gray. 
	Let $P = \{v_1,\cdots,v_m\}\subseteq V$ be the path, 
	where $m\le k$ and policy $\pi$ selects the vertices 
	$v_1,\cdots,v_m$ in this order. 
	We consider proving the above inequality for $P$ and $V\backslash P$ separately. 
	We can easily see that $\Delta(v|\psi' \cup \psi_v) = 0$ holds for all $v \in V \setminus P$ since $u$ is already activated there.
	Therefore, it is enough to prove
	\begin{equation}\label{eq:influence-recursive-inequality}
		\sum_{v \in P} \Pr(v \in E(\pi, \Phi) | \Phi \sim \psi') \left\{ \frac{2k}{k+1} \Delta(v | \psi') - \Delta(v | \psi' \cup \psi_v) \right\} \ge 0.
	\end{equation}

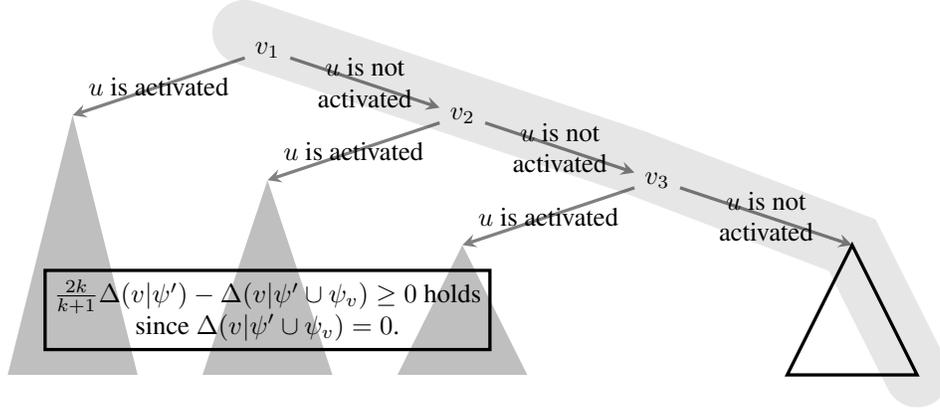
\begin{figure}
	\centering
	\newlength{\width}
	\setlength{\width}{0.4\paperwidth}
	\begin{tikzpicture}[line width=0.005\width]
		\tikzstyle{policynode}=[]
		\tikzstyle{policyedge}=[-{stealth}, gray, line width=0.005\width]
		\tikzstyle{compedge}=[-{Triangle}, line width=0.005\width, decorate, decoration={snake, amplitude=0.002\width, segment length=0.02\width}]
		\tikzstyle{triedge}=[line width=0.005\width]
		\node[policynode] (v1) at (0.0\width, 1.0\width) {$v_1$};
		\node[policynode] (v2) at (0.3\width, 0.9\width) {$v_2$};
		\node[policynode] (v3) at (0.6\width, 0.8\width) {$v_3$};
		\coordinate (c) at (-0.0\width, 0.6\width);
		\coordinate (v4) at (+0.9\width, 0.7\width);
		\coordinate (a1) at (-0.3\width, 0.9\width);
		\coordinate (a2) at (-0.0\width, 0.8\width);
		\coordinate (a3) at (+0.3\width, 0.7\width);
		\coordinate (l1) at (-0.4\width, 0.5\width);
		\coordinate (r1) at (-0.2\width, 0.5\width);
		\coordinate (l2) at (-0.1\width, 0.5\width);
		\coordinate (r2) at (+0.1\width, 0.5\width);
		\coordinate (l3) at (+0.2\width, 0.5\width);
		\coordinate (r3) at (+0.4\width, 0.5\width);
		\coordinate (l4) at (+0.8\width, 0.5\width);
		\coordinate (r4) at (+1.0\width, 0.5\width);
		\coordinate (t1) at (+0.9333\width, 0.6333\width);
		\coordinate (t2) at (+0.9667\width, 0.5667\width);
		\draw[policyedge] (v1) to node[align=center, text=black]{$u$ is activated} (a1);
		\draw[policyedge] (v2) to node[align=center, text=black]{$u$ is activated} (a2);
		\draw[policyedge] (v3) to node[align=center, text=black]{$u$ is activated} (a3);
		\draw[policyedge] (v1) to node[align=center, text=black]{$u$ is not\\activated} (v2);
		\draw[policyedge] (v2) to node[align=center, text=black]{$u$ is not\\activated} (v3);
		\draw[policyedge] (v3) to node[align=center, text=black]{$u$ is not\\activated} (v4);
		\fill[fill=gray!50] (a1) -- (l1) -- (r1) -- cycle;
		\fill[fill=gray!50] (a2) -- (l2) -- (r2) -- cycle;
		\fill[fill=gray!50] (a3) -- (l3) -- (r3) -- cycle;
		\draw[triedge] (v4) -- (l4) -- (r4) -- cycle;
		\node[align=center, draw, rectangle] at (c) {$\frac{2k}{k+1} \Delta(v | \psi') - \Delta(v | \psi' \cup \psi_v) \ge 0$ holds\\since $\Delta(v | \psi' \cup \psi_v) = 0$.};
		\draw [line width=0.1\width, opacity=0.1, line cap=round] (v1.north west) -- (v2.north west) -- (v3.north west) -- (v4.north west) -- (t1) -- (t2) -- (r4);
	\end{tikzpicture}
	\caption{A description of our proof method. We can decompose the policy tree into the path wherein $u$ is not activated and the rest.}\label{fig:compensation}
\end{figure}
	
	Now we calculate the left hand side of this inequality, which we denote by $C$.
	Since $u$ has not been activated yet in $\psi'$, all edges $(s, u)$ are dead for all $s \in \dom(\psi')$. 
	In the linear threshold model, we can define $p_i \coloneqq b({v_i u}) / (1 - \sum_{t \in V \setminus \dom(\psi')} b({tu}))$ to be the posterior probability that edge $(v, u)$ is alive under observations $\psi'$ for each $i = 1,\dots,m$.
	Now we have $\Pr( v_{i} \in E(\pi, \Phi) | \Phi \sim \psi') = \Pr(\Phi \sim \psi' \cup \psi_{v_{i}} | \Phi \sim \psi') = 1 - \sum_{j=1}^{i-1} p_j$.
	In addition, we have $\Delta(v_i | \psi') = p_i$ and $\Delta(v_i | \psi' \cup \psi_{v_i}) = p_i / (1 - \sum_{j = 1}^{i-1} p_j)$, and hence 
	\begin{equation}
		C = \sum_{i = 1}^m \left( 1 - \sum_{j = 1}^{i-1} p_j \right) \left\{ \frac{2k}{k+1} p_i - \frac{p_i}{1 - \sum_{j = 1}^{i-1} p_j} \right\}.
	\end{equation}
	In the case of $m = 1$, we have $C = (k - 1) / (k + 1) p_i \ge 0$.
	For $m \ge 2$, we obtain
	\begin{align}
	C
	&= \sum_{i = 1}^m \frac{2k}{k+1} p_i \left(1 - \sum_{j = 1}^{i-1} {p_j} \right) - \sum_{i = 1}^m p_i \\
	&= \frac{k-1}{k+1}\sum_{i = 1}^m p_i - \frac{2k}{k+1} \sum_{i = 1}^m \left( p_i\sum_{j = 1}^{i-1} {p_j} \right)\\
	&= \frac{k-1}{k+1} \left\{ \sum_{i = 1}^m p_i - \frac{2k}{k-1} \sum_{i = 1}^m \left( p_i\sum_{j = 1}^{i-1} {p_j} \right) \right\}.
	\end{align}
	The right hand side can be bounded from below as
	\begin{align}
		&\frac{k-1}{k+1} \left\{ \sum_{i = 1}^m p_i - \frac{2k}{k-1} \sum_{i = 1}^m \left( p_i\sum_{j = 1}^{i-1} {p_j} \right) \right\}\\
		&= \frac{k-1}{k+1} \left\{ {\bf 1}^\top{\bf p} - \frac{k}{k-1}{\bf p}^\top ({\bf 1}{\bf 1}^\top - \bfI){\bf p} \right\}\\
		&\ge \frac{k-1}{k+1} \left\{ {\bf 1}^\top{\bf p} - \frac{m}{m-1}{\bf p}^\top ({\bf 1}{\bf 1}^\top - \bfI){\bf p} \right\},
	\end{align}
	where $\bfp = (p_1,\dots, p_m)^\top$ and $\bfI \in \bbR^{m \times m}$ is the identity matrix.
	The inequality comes from $2 \le m \le k$ and $\bfp^\top ({\bf 1}{\bf 1}^\top - \bfI) \bfp \ge 0$.
	Since each entry of $\bfp$ represents a probability, we have ${\bf 0}\le{\bf p}\le {\bf 1}$ and $0\le{\bf 1}^\top{\bf p}\le 1$.
	From \Cref{lem:linear-threshold-inequality} proved below, we can see that this is non-negative.
	Therefore, we conclude that \eqref{eq:influence-recursive-inequality} holds. 
\end{proof}

In the above proof, we used the following lemma. 

\begin{lem}\label{lem:linear-threshold-inequality}
	Let $m \ge 2$ and $\bfp \in \bbR^m$ be an arbitrary vector such that ${\bf 0}\le{\bf p}\le {\bf 1}$ and $0\le{\bf 1}^\top{\bf p}\le 1$, then we have
	\begin{equation}
	{\bf 1}^\top{\bf p} - \frac{m}{m-1}{\bf p}^\top ({\bf 1}{\bf 1}^\top - \bfI){\bf p} \ge 0.
	\end{equation}
\end{lem}

\begin{proof}
	Let $\bfU = (\bfu_1 \cdots \bfu_m)\in\bbR^{m\times m}$ be an orthonormal matrix whose first column is 
	defined as ${\bf u}_1={\bf 1}/\sqrt{m}$; 
	we can write ${\bf p}={\bf U}{\bf q}$ 
	with some vector ${\bf q}=(q_1,\dots,q_m)^\top$. 
	Since $\bfu_1^\top \bfu_i = 0$ for all $i \neq 1$, we obtain $\bfU^\top {\bf 1} = (\sqrt{m}, 0,\dots, 0)^\top$.
	Hence the left hand side of the target inequality can be 
	rewritten as 
	\begin{align}
	{\bf 1}^\top{\bf p} - \frac{m}{m-1}{\bf p}^\top ({\bf 1}{\bf 1}^\top - \bfI){\bf p}
	&= {\bf 1}^\top \bfU \bfq - \frac{m}{m-1}{\bfq}^\top \bfU^\top ({\bf 1}{\bf 1}^\top - \bfI) \bfU \bfq \\
	&= \frac{m}{m-1}(\|{\bf q}\|_2^2 - mq_1^2) + \sqrt{m}q_1\\
	&= \sqrt{m}q_1 (1-\sqrt{m}q_1) + \frac{m}{m-1}(q_2^2+\dots+q_m^2).
	\end{align}
	Since we have 
	$0\le{\bf 1}^\top{\bf p} = \sqrt{m}q_1 \le 1$, 
	this value is non-negative. 
\end{proof}

\subsection{Proof for the Triggering Model}\label{sec:app-infmax-triggering}
\begin{proof}[Proof of \Cref{thm:bipartite_ratio}]
	The outline of the proof for the triggering model is the same as the one for the linear threshold model. In the case of the triggering model, we can write $C$ as follows: 
	\begin{equation}
		C = \sum_{i=1}^m \Pr \left( \bigwedge_{j=1}^{i-1} X_j = 0 \middle| \Phi \sim \psi' \right) \left\{ \frac{2k}{k+1} \Pr\left(X_i=1 | \Phi \sim \psi' \right) - \Pr\left(X_i=1 \middle| \Phi \sim \psi', \bigwedge_{j=1}^{i-1} X_j = 0 \right) \right\}, 
	\end{equation}
	where $X_i$ is an event in which edge $(v_{i}, u)$ is alive. Different from the linear threshold model, we cannot express $C$ explicitly with parameters.
	Hence we define 
	\begin{align}
		p_i &\coloneqq \Pr(X_i=1 | \Phi \sim \psi') \quad \text{for } i = 1, \cdots, m, \\
	a_i &\coloneqq \textstyle {\Pr\left(X_i=1\wedge\left\{\bigwedge_{j=1}^{i-1}X_i=0\right\} \middle| \Phi \sim \psi' \right)} \quad \text{for } i = 1, \cdots, m, \\
	\text{and} \quad 
		h_i &\coloneqq \textstyle \Pr\left(\left\{\bigwedge_{j=1}^{i}X_i=0\right\} \middle| \Phi \sim \psi' \right) \quad \text{for } i = 0,\cdots,m.
	\end{align}
	With these definitions, we can calculate $C$ as
	\begin{align}
	C
	&= \sum_{i=1}^m h_{i-1} \left( \frac{2k}{k+1} p_i - \frac{a_i}{h_{i-1}} \right)\\
	&= \frac{2k}{k+1} \sum_{i=1}^m p_i h_{i-1} - \sum_{i=1}^m a_i.
	\end{align}
	Our goal is to prove that
	\begin{equation}
	\frac{2k}{k+1} \sum_{i=1}^m p_i h_{i-1} - \sum_{i=1}^m a_i \ge 0.
	\end{equation}
	Note that we have 
	\begin{align}
	0\le a_i \le p_i \le 1
	& &
	\text{and}
	& &
	0\le h_i \le 1
	& &
	\text{for $i=1,\dots,m$}, 
	\end{align}
	where $a_1=p_1$ and $h_0=1$. 
	Therefore, if $m = 1$, we have
	\begin{align}
		\frac{2k}{k+1} \sum_{i=1}^m p_i h_{i-1} - \sum_{i=1}^m a_i
		&= \frac{2k}{k+1} p_1 h_0 - a_1\\
		&= \frac{k-1}{k+1} p_1\\
		&\ge 0. 
	\end{align}
	Furthermore, 
	it holds that 
	\[
	\textstyle
	h_i + \sum_{j=1}^i a_j = 
	\Pr\left(\left\{\bigwedge_{j=1}^{i}X_i=0\right\}\right)
	+
	\sum_{j=1}^i
	{\Pr\left(X_i=1\wedge\left\{\bigwedge_{j=1}^{i-1}X_i=0\right\}\right)}
	=1
	\]
	for $i=0,\dots,m$, 
	where $\sum_{j=1}^0a_j=0$.  
	By combining this equality with $0\le h_i \le 1$, 
	we obtain 
	\[
	0 \le \sum_{j=1}^i a_j \le 1
	\]
	for $i=0,\dots,m$. 
	With these inequalities, the LHS of the target inequality can be lower-bounded as  
	\begin{align}
	\frac{2k}{k+1} \sum_{i=1}^m p_i h_{i-1} - \sum_{i=1}^m a_i
	&=
	\frac{2k}{k+1} \sum_{i=1}^m p_i
	\left( 
	1-\sum_{j=1}^{i-1} a_j
	\right)  - \sum_{i=1}^m a_i \\
	&\ge
	\frac{2k}{k+1} \sum_{i=1}^m a_i
	\left( 
	1-\sum_{j=1}^{i-1} a_j
	\right)  - \sum_{i=1}^m a_i \\
	&=
	\frac{k-1}{k+1} \sum_{i=1}^m a_i 
	-
	\frac{2k}{k+1} \sum_{i>j} a_ia_j \\
	&=
	\frac{k-1}{k+1} \left( {\bf 1}^\top\bfa 
	- 
	\frac{k}{k-1} \bfa^\top ({\bf 1}{\bf 1}^\top - \bfI) \bfa\right)\\
	&\ge
	\frac{k-1}{k+1} \left( {\bf 1}^\top\bfa 
	- 
	\frac{m}{m-1} \bfa^\top ({\bf 1}{\bf 1}^\top - \bfI) \bfa\right),
	\end{align}
	which is non-negative from \Cref{lem:linear-threshold-inequality};
	this completes the proof as with the case of the linear threshold model. 
\end{proof}

\subsection{Example for the Case of General Graphs}\label{sec:app-infmax-general}
In this subsection, we provide a problem instance of a general graph in which the adaptive submodularity ratio can be very small. 

Before that, we briefly describe the problem setting of adaptive influence maximization in general graphs.
Let $G = (V', A)$ be a general directed graph and $V \subseteq V'$ be a set of vertices that can be selected.
At each step, the algorithm selects one vertex $v \in V$, then the influence spreads from $v$ according to some stochastic diffusion process such as the independent cascade model or the linear threshold model.
After that, the algorithm observes the diffusion from this vertex $v$ under some feedback model.
This problem includes the bipartite influence maximization as a special case where $G = (V \cup U, A)$ is a directed bipartite graph with $A \subseteq V \times U$ and $w(v) = 0$ for all $v \in V$.

There are two standard feedback models, both of which are proposed by \citet{GK11}.
Note that these two feedback models are equivalent in bipartite graphs.
In the first feedback model called the \textit{myopic feedback model}, the algorithm observes the states of all edges outgoing from $v$.
\citet{GK11} proved that the adaptive submodularity does not hold in this case by giving a simple example.
This analysis can be applied to both the independent cascade and linear threshold models.
With this example instance, we can readily see that the adaptive submodularity ratio can be very small under the myopic feedback model. 
These facts imply that the myopic feedback model is typically too harsh to deal with. 

In the second feedback model called the \textit{full-adoption feedback model}, the algorithm observes the states of all edges outgoing from any vertex $u \in R(v)$ when selecting $v$, 
where $R(v)$ is the set of all vertices reachable from $v$ only through live edges.
\citet{GK11} showed that, 
even if graphs are general (non-bipartite), 
the objective function satisfies adaptive submodularity 
under 
the independent cascade model with the full-adaption feedback. 

Below we show that, 
even under the linear threshold model with the full-adoption feedback, 
the adaptive submodularity ratio can be arbitrarily small 
if the graph is non-bipartite. 
This fact implies that the assumption of bipartiteness,  
which we imposed to obtain the bound on the adaptive submodularity ratio, is almost inevitable. 

\begin{exmp}
    Let $G$ be a directed graph with vertices $V = \{v_1,\dots,v_\ell\} \cup \{u_0,u_1,\dots,u_\ell\}$ 
    and directed edges $A = \{(u_{i-1}, u_i) \mid i = 1, \dots, \ell\} \cup \{(v_i, u_i) \mid i = 1, \dots, \ell\}$.
    Let $w$ be the vertex weight such that $w(v) = 1$ for all $v \in V$.
    We consider the following linear threshold model: 
    for each $i \in [\ell]$, only one of the two edges, $(v_i, u_i)$ and $(u_{i-1}, u_i)$, entering $u_i$ is alive with probability $\epsilon$ and $1-\epsilon$, respectively.

    Let $\pi$ be a policy defined as follows. $\pi$ first selects $u_0$. 
    Then the realized states of some edges are revealed under the full-adoption feedback model and we can observe which vertices are activated. 
    If $u_\ell$ is activated, $\pi$ stops. 
    Otherwise, there exists some $i \in [\ell]$ such that 
    $u_{i-1}$ is activated but $u_i$ is not. 
    Then $\pi$ proceeds to select $v_i$.
    Repeat this procedure until $u_\ell$ is activated.
	The graph and policy are illustrated in \Cref{fig:infmax_instance2}.

	First we consider the probability $\Pr(v_i \in E(\pi, \Phi))$ for each $i \in [\ell]$. 
    We can see that $\pi$ selects $v_i$ if and only if the edge $(u_{i-1}, u_i)$ is dead, which yields $\Pr(v_i \in E(\pi, \Phi)) = \epsilon$. 
    We can easily confirm that $\pi$ finally activates all $u_0,\dots, u_\ell$ for every realization and each $v_i$ is selected with probability $\epsilon$, therefore $\Delta(\pi|\emptyset) = \ell + 1 + \epsilon \ell$. 
    On the other hand, the numerator of the definition of the adaptive submodularity ratio can be calculated as follows.
	The expected marginal gain of $v_i$ is
    \begin{align}
        \Delta(v_i|\emptyset) = 1 + \sum_{j = i}^{\ell} \epsilon (1 - \epsilon)^{j-i} 
        = 2 - (1 - \epsilon)^{\ell-i+1}.
    \end{align}
Similarly, we have $\Delta(u_0|\emptyset) = \frac{1}{\epsilon} \{ 1 - (1 - \epsilon)^{\ell+1} \}$. Finally, we can compute the adaptive submodularity ratio as
\begin{align}
    \gamma_{\emptyset, \ell}
    &\le \frac{\sum_{v \in V} \Pr(v \in E(\pi, \Phi) ) \Delta(v | \emptyset) }{\Delta(\pi | \emptyset)}\\
    &= \frac{\frac{1}{\epsilon} \{ 1 - (1 - \epsilon)^{\ell+1} \} + \epsilon \sum_{i=1}^\ell (2 - (1 - \epsilon)^{\ell-i+1})}{\ell + \epsilon \ell + 1}\\
    &\le \frac{\frac{1}{\epsilon} + 2 \epsilon \ell}{\ell + \epsilon \ell + 1}
\end{align}
By setting $\epsilon = 1 / \sqrt{\ell}$ and taking $\ell \to \infty$, we can see $\gamma_{\emptyset, \ell} \to 0$. 
To conclude, the adaptive submodularity ratio can become arbitrarily small if the graph is non-bipartite.
\end{exmp}

\begin{figure}[t]
\centering
\begin{tabular}{cc}
	\begin{minipage}{0.35\hsize}
	\centering
	\begin{tikzpicture}
		\newlength{\basewidth}
		\setlength{\basewidth}{0.4\paperwidth}
		\tikzstyle{vertex}=[circle, fill, inner sep=0.01\basewidth]
		\tikzstyle{arc}=[-{stealth}, line width=0.005\basewidth]
		\node[vertex, label=below:$u_0$]     (u0) at (0.0\basewidth, 0.0\basewidth) {};
		\node[vertex, label=below:$u_1$]     (u1) at (0.1\basewidth, 0.0\basewidth) {};
		\node[vertex, label=below:$u_2$]     (u2) at (0.2\basewidth, 0.0\basewidth) {};
		\node[vertex, label=below:$u_{\ell-1}$] (u3) at (0.5\basewidth, 0.0\basewidth) {};
		\node[vertex, label=below:$u_{\ell}$]   (u4) at (0.6\basewidth, 0.0\basewidth) {};
		\node[vertex, label=above:$v_1$]     (v1) at (0.1\basewidth, 0.1\basewidth) {};
		\node[vertex, label=above:$v_2$]     (v2) at (0.2\basewidth, 0.1\basewidth) {};
		\node[vertex, label=above:$v_{\ell-1}$] (v3) at (0.5\basewidth, 0.1\basewidth) {};
		\node[vertex, label=above:$v_\ell$]     (v4) at (0.6\basewidth, 0.1\basewidth) {};
		\coordinate (ul) at (0.3\basewidth, 0.0\basewidth) {};
		\coordinate (ur) at (0.4\basewidth, 0.0\basewidth) {};
		\draw[arc] (v1) -- (u1);
		\draw[arc] (v2) -- (u2);
		\draw[arc] (v3) -- (u3);
		\draw[arc] (v4) -- (u4);
		\draw[arc] (u0) -- (u1);
		\draw[arc] (u1) -- (u2);
		\draw[arc] (u2) -- (ul);
		\draw[arc] (ur) -- (u3);
		\draw[arc] (u3) -- (u4);
		\draw[dotted, line width=0.005\basewidth] (ul) -- (ur);
	\end{tikzpicture}\\
		\small (a) graph $G = (V, A)$
	\end{minipage}

	\begin{minipage}{0.65\hsize}
	\centering
	\begin{tikzpicture}
		\setlength{\basewidth}{0.4\paperwidth}
		\tikzstyle{policynode}=[]
		\tikzstyle{policyedge}=[-{stealth}, line width=0.005\basewidth]
		\tikzstyle{vertex}=[circle, fill, inner sep=0.01\basewidth]
		\tikzstyle{arc}=[-{stealth}, line width=0.005\basewidth]
		\node[policynode]     (u0) at (0.1\basewidth, 0.0\basewidth) {$u_0$};
		\node[policynode]     (v1) at (0.3\basewidth, 0.0\basewidth) {$v_1$};
		\node[policynode]     (v2) at (0.6\basewidth, 0.0\basewidth) {$v_2$};
		\node[policynode]     (v3) at (0.9\basewidth, 0.0\basewidth) {$v_{\ell-1}$};
		\node[policynode]     (v4) at (1.2\basewidth, 0.0\basewidth) {$v_{\ell}$};
		\coordinate (f1) at (0.3\basewidth, 0.2\basewidth);
		\coordinate (f2) at (0.6\basewidth, 0.2\basewidth);
		\coordinate (f3) at (0.9\basewidth, 0.2\basewidth);
		\coordinate (f4) at (1.2\basewidth, 0.2\basewidth);
		\coordinate (t1) at (0.68\basewidth, 0.20\basewidth);
		\coordinate (t2) at (0.68\basewidth, 0.15\basewidth);
		\coordinate (t3) at (0.68\basewidth, 0.05\basewidth);
		\coordinate (t4) at (0.68\basewidth, 0.00\basewidth);
		\coordinate (s1) at (0.82\basewidth, 0.20\basewidth);
		\coordinate (s2) at (0.82\basewidth, 0.15\basewidth);
		\coordinate (s3) at (0.82\basewidth, 0.05\basewidth);
		\coordinate (s4) at (0.82\basewidth, 0.00\basewidth);
		\draw[policyedge] (u0) to node[align=center, below]{\tiny $(u_0,u_1)$\\[-1ex]\tiny is dead} (v1);
		\draw[policyedge] (v1) to node[align=center, below]{\tiny $(u_1,u_2)$\\[-1ex]\tiny is dead} (v2);
		\draw[policyedge] (v3) to node[align=center, below]{\tiny $(u_{\ell-1},u_\ell)$\\[-1ex]\tiny is dead} (v4);
		\draw[policyedge] (u0) to node[align=center, above]{\tiny $(u_0,u_1)$\\[-1ex]\tiny is alive} (f1);
		\draw[policyedge] (v1) to node[align=center, below, xshift=-0.13\basewidth]{\tiny $(u_1,u_2)$\\[-1ex]\tiny is alive} (f2);
		\draw[policyedge] (v3) to node[align=center, below, xshift=-0.13\basewidth]{\tiny $(u_{\ell-1},u_\ell)$\\[-1ex]\tiny is alive} (f4);
		\draw[policyedge] (f1) to node[align=center, above]{\tiny $(u_1,u_2)$\\[-1ex]\tiny is dead} (v2);
		\draw[policyedge] (f3) to node[align=center, above]{\tiny $(u_{\ell-1},u_\ell)$\\[-1ex]\tiny is dead} (v4);
		\draw[policyedge] (f1) to node[align=center, above]{\tiny $(u_1,u_2)$\\[-1ex]\tiny is alive} (f2);
		\draw[policyedge] (f3) to node[align=center, above]{\tiny $(u_{\ell-1},u_\ell)$\\[-1ex]\tiny is alive} (f4);
		\draw[dotted, line width=0.005\basewidth] (0.7\basewidth, 0.1\basewidth) -- (0.8\basewidth, 0.1\basewidth);

		\draw[line width=0.005\basewidth] (f2) -- (t1);
		\draw[line width=0.005\basewidth] (f2) -- (t2);
		\draw[line width=0.005\basewidth] (v2) -- (t3);
		\draw[line width=0.005\basewidth] (v2) -- (t4);
		\draw[policyedge] (s1) -- (f3);
		\draw[policyedge] (s2) -- (f3);
		\draw[policyedge] (s3) -- (v3);
		\draw[policyedge] (s4) -- (v3);
	\end{tikzpicture}\\
		\small (b) policy $\pi$
	\end{minipage}
\end{tabular}
	\caption{An instance with a non-bipartite graph such that the adaptive submodularity ratio can be arbitrarily small. Since the space is limited, nodes of $\pi$ that have the same subtree are indicated by a single node.}\label{fig:infmax_instance2}
\end{figure}
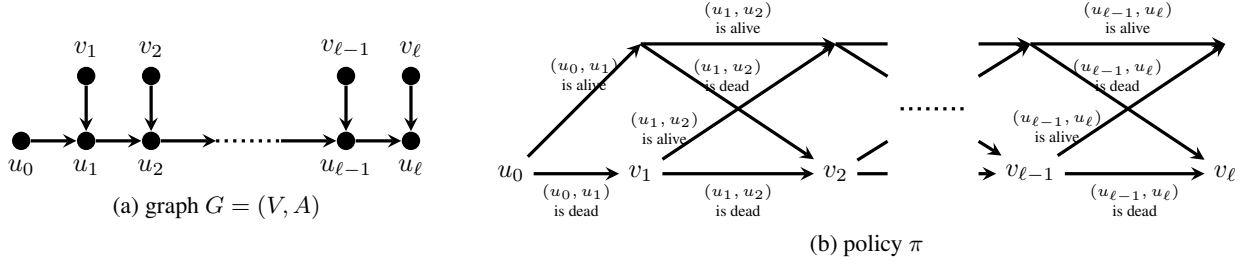

\section{Proof for Adaptive Feature Selection}\label{sec:app-feature-ratio}
\begin{proof}[Proof of \Cref{thm:stochastic-ratio}]
	Let $\psi$ be any partial realization and $\pi \in \Pi_k$ be any policy of height at most $k$.
	Fix an arbitrary subset $\psi' \subseteq \psi$.
	Note that 
	we have $f(\dom(\psi),\phi)=f(\dom(\psi),\phi')$ 
	for any $\phi,\phi'\supseteq\psi$ 
	due to the assumption that $f(S, \phi)$ depends only on $(\phi(v))_{v \in S}$; 
	considering this, 
	we abuse the notation and define 
	$f(\psi) \coloneqq f(\dom(\psi), \phi)$ for any $\phi\supseteq\psi$. 
	Let $\psi_v$ be the partial realization just before $v$ is selected in $\pi$.
	If there are multiple partial realizations $\psi$ such that $\pi(\psi) = v$, we can duplicate $v$ and consider them to be different elements.
	Now we can transform the numerator of the adaptive submodularity ratio as
	\begin{align}
	&\sum_{v \in V} \Pr(v \in E(\pi, \Phi) | \Phi \sim \psi') \Delta(v|\psi')\\
	&=\sum_{v \in V} \Pr(v \in E(\pi, \Phi) | \Phi \sim \psi') \bbE \left[ f(\psi' \cup \{(v, \Phi(v))\}) - f(\psi') \middle| \Phi \sim \psi' \right] \\
	&=\sum_{v \in V} \Pr(v \in E(\pi, \Phi) | \Phi \sim \psi') \sum_{y \in \calY} \Pr(\Phi(v) = y | \Phi \sim \psi') \Bigl\{ f(\psi' \cup \{(v, y)\}) - f(\psi') \Bigr\} \\
	&=\sum_{v \in V} \Pr(v \in E(\pi, \Phi) | \Phi \sim \psi') \sum_{y \in \calY} \Pr(\Phi(v) = y | \Phi \sim \psi' \cup \psi_v) \Bigl\{ f(\psi' \cup \{(v, y)\}) - f(\psi') \Bigr\} \tag{due to the independence of $\phi(v)$ from $(\phi(u))_{u \in \dom(\psi_v)}$} \\
	&=\sum_{v \in V} \Pr(v \in E(\pi, \Phi) | \Phi \sim \psi') \bbE \left[ f(\psi' \cup \{(v, \Phi(v))\}) - f(\psi') \middle| \Phi \sim \psi' \cup \psi_v \right] \\
	&=\sum_{v \in V} \Pr(\Phi \sim \psi' \cup \psi_v | \Phi \sim \psi') \bbE \left[ f(\dom(\psi') \cup \{v\}, \Phi\}) - f(\dom(\psi'), \Phi) \middle| \Phi \sim \psi' \cup \psi_v \right] \\
	&= \bbE \left[ \sum_{v \in E(\pi, \Phi)} \Bigl\{ f(\dom(\psi') \cup \{v\}, \Phi) - f(\dom(\psi'), \Phi) \Bigr\} \middle| \Phi \sim \psi' \right].
	\end{align}
	
	From the above equality, we get 
	\begin{align}
	&\min_{\phi \sim \psi} \gamma_{\dom(\psi), k}^\phi \Delta(\pi|\psi')\\
	&= \min_{\phi \sim \psi} \gamma_{\dom(\psi), k}^\phi \bbE \left[ f(\dom(\psi') \cup E(\pi, \Phi), \Phi) - f(\dom(\psi'), \Phi) \middle| \Phi \sim \psi' \right]\\
	&\le \bbE \left[ \gamma_{\dom(\psi), k}^\Phi \Bigl\{ f(\dom(\psi') \cup E(\pi, \Phi), \Phi) - f(\dom(\psi'), \Phi) \Bigr\} \middle| \Phi \sim \psi' \right]\\
	&\le \bbE \left[ \sum_{v \in E(\pi, \Phi)} \Bigl\{ f(\dom(\psi') \cup \{v\}, \Phi) - f(\dom(\psi'), \Phi) \Bigr\} \middle| \Phi \sim \psi' \right] \tag{From the definition of submodularity ratio}\\
	&= \sum_{v \in V} \Pr(v \in E(\pi, \phi) | \Phi \sim \psi') \Delta(v|\psi').
	\end{align}
	This inequality holds for any $\psi$ and $\pi \in \Pi_k$. To conclude, we obtain $\gamma_{\psi, k} \ge \min_{\phi \sim \psi} \gamma^\phi_{\dom(\psi), k}$.
\end{proof}

To prove \Cref{cor:stochastic}, we use the following bound on the submodularity ratio provided by \citet{Das2011}.
\begin{thm}[Adopted from {\citep[Lemma 2.4]{Das2011}}]\label{thm:daskempe}
	Assume each column of $\bfA(\phi)$ is normalized. Then
	\begin{equation}
	\gamma_{U, k} \ge \min_{S \subseteq V \colon |S| \le k + |U|} \lambda_{\min}(\bfA(\phi)_S^\top \bfA(\phi)_S). 
	\end{equation}
\end{thm}

\begin{proof}[Proof of \Cref{cor:stochastic}]
	From the definition,
	we can see $f(S, \phi)$ depends only on selected columns $(\phi(v))_{v \in S}$ and not on the other columns $(\phi(v))_{v \in V \setminus S}$.
	
	We can show
	\begin{align}
	f(S, \phi) 
	&= \| \bfb - \bfzero \|^2 - \min_{\supp(\bfw) \subseteq S} \| \bfb - \bfA(\phi) \bfw \|^2
	\\
	&\le \| \bfb - \bfzero \|^2 - \min_{\supp(\bfw) \subseteq T} \| \bfb - \bfA(\phi) \bfw \|^2 = f(T, \phi)
	\end{align}
	for all $S \subseteq T$.
	From this property, called strong adaptive monotonicity, for any partial realization $\psi$ and $v \in V \setminus \dom(\psi)$, we obtain
	\begin{align}
	\Delta(v | \psi) 
	={}& \bbE[f(\dom(\psi) \cup \{v\}, \Phi) - f(\dom(\psi), \Phi) | \Phi \sim \psi] \\
	&\ge 0,
	\end{align}
	from which the adaptive monotonicity of $f$ w.r.t.\ $p$ follows.
	
	By applying \Cref{thm:stochastic-ratio}, we obtain
	\begin{equation}
	\gamma_{\psi, k} \ge \min_{\phi \sim \psi} \gamma^\phi_{\dom(\psi), k}, 
	\end{equation}
	where $\gamma^\phi_{X, k}$ is the submodularity ratio of $f( \cdot, \phi)$ for realization $\phi$. From \Cref{thm:daskempe}, we obtain the following lower bound: 
	\begin{equation}
	\gamma^\phi_{\dom(\psi), k} \ge \min_{S \subseteq V \colon |S| \le k + |\psi|} \lambda_{\min}(\bfA(\phi)_S^\top \bfA(\phi)_S).
	\end{equation}
	Finally, we have
	\begin{align}
	\gamma_{\ell, k}
	&= \min_{\psi \colon |\psi| \le \ell} \gamma_{\psi, k}\\
	&\ge \min_{\psi \colon |\psi| \le \ell} \min_{\phi} \gamma^\phi_{\dom(\psi), k}\\
	&\ge \min_{\psi \colon |\psi| \le \ell} \min_{\phi} \min_{S \subseteq V \colon |S| \le k + |\psi|} \lambda_{\min}(\bfA(\phi)_S^\top \bfA(\phi)_S)\\
	&= \min_{\phi} \min_{S \subseteq V \colon |S| \le k + \ell} \lambda_{\min}(\bfA(\phi)_S^\top \bfA(\phi)_S).
	\end{align}
\end{proof}

\subsection{Proof for the Adaptivity Gap}\label{sec:app-feature-gap}
\begin{proof}[Proof of \Cref{prop:feature-gap}]
We can readily confirm that the objective function can be rewritten as follows:  
\begin{equation}
f(S, \phi) = (\bfA(\phi)_S^\top \bfb)^\top (\bfA(\phi)_S^\top \bfA(\phi)_S)^{+} (\bfA(\phi)_S^\top \bfb).
\end{equation}
For any $S \subseteq V$ such that $|S| \le k$, we have
\begin{align}
	\bbE [f(S, \Phi)]
	&= \bbE \left[ (\bfA(\Phi)_S^\top \bfb)^\top (\bfA(\Phi)_S^\top \bfA(\Phi)_S)^+ (\bfA(\Phi)_S^\top \bfb) \right]\\
	&\ge \bbE \left[ \lambda_{\min}((\bfA(\Phi)_S^\top \bfA(\Phi)_S)^+) \|\bfA(\Phi)_S^\top \bfb\|_2^2 \right]\\
	&\ge \bbE \left[ \frac{ \|\bfA(\Phi)_S^\top \bfb\|_2^2}{ \max_{\phi} \max_{T \subseteq V \colon |T| \le k} \lambda_{\max}(\bfA(\phi)_T^\top \bfA(\phi)_T)} \right]\\
	&= \frac{\bbE \left[  \|\bfA(\Phi)_S^\top \bfb\|_2^2 \right]}{ \max_{\phi} \max_{T \subseteq V \colon |T| \le k} \lambda_{\max}(\bfA(\phi)_T^\top \bfA(\phi)_T)}\\
	&= \frac{\sum_{v \in S} \bbE \left[  (\bfA(\Phi)_v^\top \bfb)^2 \right]}{ \max_{\phi} \max_{T \subseteq V \colon |T| \le k} \lambda_{\max}(\bfA(\phi)_T^\top \bfA(\phi)_T)}\\
	&= \frac{\sum_{v \in S} \bbE[ f(\{v\}, \Phi) ]}{ \max_{\phi} \max_{T \subseteq V \colon |T| \le k} \lambda_{\max}(\bfA(\phi)_T^\top \bfA(\phi)_T)}.
\end{align}
From this inequality, we can bound the supermodularity ratio $\beta_{\emptyset, k}$ of $\bbE_\Phi[f(\cdot, \Phi)]$ as
\begin{equation}
	\beta_{\emptyset, k} \ge \frac{1}{ \max_{\phi} \max_{S \subseteq V \colon |S| \le k} \lambda_{\max}(\bfA(\phi)_S^\top \bfA(\phi)_S)}.
\end{equation}
Plugging it and the inequality of \Cref{cor:stochastic} into \Cref{thm:gap}, we obtain
\begin{equation}
	\gap_k \ge \frac{\min_{\phi} \min_{S \subseteq V \colon |S| \le k} \lambda_{\min}(\bfA(\phi)_S^\top \bfA(\phi)_S)}{ \max_{\phi} \max_{S \subseteq V \colon |S| \le k} \lambda_{\max}(\bfA(\phi)_S^\top \bfA(\phi)_S)}.
\end{equation}
\end{proof}

\section{Counterexample to the Statement of \citep{Kusner14}}\label{sec:counter-kusner}
\citet{Kusner14} has defined \textit{approximate adaptive submodularity} as follows:
\begin{defn}[{Adopted from \citep[Definition 2]{Kusner14}}]
	A set function $f \colon 2^V \times \mathcal{Y}^V \to \mathbb{R}$ and a distribution $p$ on $\mathcal{Y}^V$ is \textit{approximately adaptive submodular} if for any subrealization $\psi$ such that $p(\psi) > 0$ and any $S \subseteq V \setminus \mathrm{range}(\psi)$, we have 
	\begin{equation}
	\sum_{v \in S} \Delta(v | \psi) \ge \gamma \Delta(S | \psi), 
	\end{equation}
	where $\gamma \in [0, 1]$ represents the submodularity ratio of the non-adaptive function. 
\end{defn}
Below we present a counterexample to the statement of \citep{Kusner14}, 
which says that a bounded $\gamma$ yields a bounded approximation ratio of the adaptive greedy algorithm. 

Let $\mathcal{Y} = \{0, 1,\dots, M-1\}$ be the set of all possible states and $V = \{u\} \cup \{z_i \mid i \in [k]\} \cup \{v^y_i \mid i \in [k-1], ~ y \in \mathcal{Y} \}$ be the ground set.
We define $f \colon 2^V \times \mathcal{Y}^V \to \mathbb{R}$ as follows:
\begin{equation}
	f(S, \phi) = |S \cap \{u\}| + (1+\epsilon)|S \cap \{z_1,\dots,z_k\}| + M \sum_{y \in \mathcal{Y}, i \in [k-1]} \mathbf{1}_{\{ \phi(v_i^y) = 1 ~ \text{and} ~ v_i^y \in S\}}, 
\end{equation}
where $\epsilon > 0$ is any small constant.
Note that this function is normalized and adaptive monotone. 
For each $y \in \calY$, we define $\phi_y$ as $\phi_y(u) = y$, $\phi_y(z_i) = 0$ for each $i \in [k]$, $\phi_y(v_i^y) = 1$ for each $i \in [k-1]$, and $\phi_y(v_i^{y'}) = 0$ for each $y' \in \calY \setminus \{y\}$ and $i \in [k-1]$.
Let $p$ be a distribution defined as 
\begin{equation}
p(\phi) = 
\begin{cases}
	\frac{1}{|\mathcal{Y}|} & \text{if} ~ \phi = \phi_y ~ \text{for some $y \in \calY$} \\
0 & \text{otherwise}.
\end{cases}
\end{equation}
It is easy to see that $f$ is approximately adaptive submodular with $\gamma=1$ w.r.t. $p$ because $\Delta( \cdot | \psi)$ is a linear function for any subrealization $\psi$.
Note that $f$ is not adaptive submodular w.r.t.\ $p$ because $\Delta(v_1^1 | \emptyset) = 1 < M = \Delta(v_1^1 | \{(u, 1)\})$. 

\citet{Kusner14} stated that the adaptive greedy algorithm achieves $(1 - \mathrm{e}^{- \gamma})$-approximation for any normalized, adaptive monotone, and approximately adaptive submodular function.
However, the adaptive greedy algorithm achieves only $(1+\epsilon)/M$-approximation for the above $f$ and $p$ as is explained below. The adaptive greedy algorithm selects $z_1,\dots,z_k$ since their expected marginal gain is $1 + \epsilon$ and the expected marginal gain of other elements is $1$.
On the other hand, the optimal policy first selects $u$ and proceeds to select $\{v^{\phi(u)}_1,\dots,v^{\phi(u)}_{k-1}\}$ according to the observed $\phi(u)$.
The adaptive greedy policy achieves $k(1+\epsilon)$ and the optimal policy achieves $1 + (k-1)M$.  
Thus the approximation ratio gets close to $(1+\epsilon)/M$ as $k$ increases, 
and it can be arbitrarily small since the number of possible states, $M$, is not bounded.   
Namely, 
even if $\gamma$ is bounded by a constant, 
the approximation guarantee of the adaptive greedy algorithm can become arbitrarily bad in general, 
which contradicts the statement of~\citep{Kusner14}. 

\section{About Comparison with \citep{YGO17}}
\subsection{Proof for Comparison with \citep{YGO17}}\label{subsec:proof-yong}
\begin{proof}[Proof of \Cref{prop:curvature-ratio-bound}]
	From the definition of $\zeta^*$, we have $\zeta^* \Delta(v|\psi) \ge \Delta(v|\psi')$ for any $\psi \subseteq \psi'$ and $v \in V \setminus \dom(\psi')$.
	It is enough to show $\frac{1}{\zeta^*} \Delta(\pi | \psi) \le \sum_{v \in V} \Pr(v \in E(\pi, \phi)) \Delta(v | \psi)$ for arbitrary $\psi \subseteq \psi'$ and $\pi$.
	Let $\psi_v$ be the partial realization just before $v$ is selected in $\pi$.
	If there are multiple partial realizations $\psi$ such that $\pi(\psi) = v$, we can duplicate $v$ and take them to be different elements.
	Then we can write $\Delta(\pi | \psi) = \sum_{v \in V} \Pr(v \in E(\pi, \phi)) \Delta(v | \psi_v)$.
	By applying the bound of weak adaptive submodularity, we have
	\begin{align}
	\Delta(\pi | \psi) &= \sum_{v \in V} \Pr(v \in E(\pi, \phi)) \Delta(v | \psi \cup \psi_v)\\
	&\le \zeta^* \sum_{v \in V} \Pr(v \in E(\pi, \phi)) \Delta(v | \psi),
	\end{align}
	which implies the statement.
\end{proof}

From this proposition, we can see that \Cref{thm:adaptive-greedy} is stronger than the result of \citep{YGO17} as follows. 
They showed that the adaptive greedy algorithm 
is guaranteed to achieve $(1 - \exp(- \ell / (\zeta^* k)))$-approximation in \citep[{Theorem 1}]{YGO17}.
From \Cref{prop:curvature-ratio-bound}, we always have $(1 - \exp(- \ell / (\zeta^* k))) \le (1 - \exp(- \gamma_{\psi, k}\ell /  k))$.

\subsection{Counterexample to \citep[{Proposition 2}]{YGO17}}\label{sec:counter-yong}
In this subsection we provide an instance of group-based active diagnosis in which the weak adaptive submodularity cannot give a bound of the approximation ratio of the adaptive greedy algorithm.

The formal problem statement of group-based active diagnosis can be described as follows. 
We have set $V$ of tests and set $\calY$ of their possible outcomes. 
There are two random variables that uniquely specify the outcome of each test: the state $x$ and the mode $q$.
Let $\calX$ be the set of all possible states and $\calQ$ the set of all possible modes. 
We know the prior joint distribution $p(x, q)$ of $x$ and $q$, but does not know their true values.
Let $\mu(v, x, q) \in \calY$ be the unique outcome of test $v$ when the true state is $x$ and the true mode is $q$.
We aim to determine $x$ by sequentially conducting several tests out of $V$.

\citet{YGO17} formulated this problem as the problem of maximizing the following objective function: 
\begin{equation}
	f(S, (x,q)) = 1 - \sum_{x' \in \calX \colon \exists q' \in \calQ, ~ \forall v \in S, ~ \mu(v, x', q') = \mu(v, x, q')} \sum_{q'' \in \calQ} p(x', q''), 
\end{equation}
where the first summation is about all possible $x' \in \calX$ under the outcomes of tests $S$ made so far.
Proposition 2 of \citep{YGO17} claims that this objective function is $\zeta$-weakly adaptive submodular for 
\begin{equation}
	\zeta \le \frac{|\calQ|}{\min_{x \in \calX, q \in \calQ} p(x,q)}. 
\end{equation}
However, it does not hold in the following example.

\begin{exmp}
	Let $\calX = \{x_1, x_2\}$ be the set of states and $\calQ = \{q_1, q_2, q_3\}$ the set of modes. 
	For each $x \in \calX$ and $q \in \calQ$, we assume $p(x, q) = \frac{1}{6}$. 
	We consider two actions $v_1$ and $v_2$, which yield the unique outcome out of $\calY = \{+1, -1\}$ indicated in \Cref{table:outcome} for each state $x \in \calX$ and mode $q \in \calQ$.
	
	\begin{table}[htbp]
		\centering
		\caption{Outcome}
		\label{table:outcome}
		\begin{tabular}{|c|c|c|}
			\hline
			$(x, q)$ & $\mu(v_1, x, q)$ & $\mu(v_2, x, q)$\\
			\hline
			$(x_1, q_1)$ & $+1$ & $+1$\\
			$(x_1, q_2)$ & $+1$ & $-1$\\
			$(x_1, q_3)$ & $-1$ & $+1$\\
			$(x_2, q_1)$ & $+1$ & $+1$\\
			$(x_2, q_2)$ & $+1$ & $-1$\\
			$(x_2, q_3)$ & $-1$ & $-1$\\
			\hline
		\end{tabular}
	\end{table}
	
	We first consider the expected marginal gain obtained by performing action $v_2$ at the beginning. In this situation, performing $v_2$ yields outcome $+1$ or $-1$ with probability $1/2$. If the outcome is $+1$, we can reject neither $x_1$ nor $x_2$. This is the case for outcome $-1$. Thus we have $\Delta(v_2 | \emptyset) = 0$.
	
	Next we assume the algorithm performs $v_1$ at the beginning and obtains the outcome of $-1$, i.e., $\psi = \{(v_1, -1)\}$. Now the possible pairs of the state and the mode are only $(x_1, q_3)$ and $(x_2, q_3)$. By performing action $v_2$, we obtain the outcome $+1$ or $-1$ with probability $\frac{1}{2}$ and reject $x_2$ or $x_1$, respectively. Thus the expected marginal gain is $\Delta(v_2 | \psi) = \frac{1}{2} \bbP[x_2] + \frac{1}{2} \bbP[x_1] = \frac{1}{2}$.
	
	From the definition of $\zeta$, we must have $\Delta(v_2 | \psi) \le \zeta \Delta(v_2 | \emptyset)$, but no finite $\zeta$ satisfies this inequality. This contradicts Proposition 2 in \citep{YGO17}, which claims $\zeta$ is finite.
\end{exmp}

\subsection{Comparison in Adaptive Influence Maximization}\label{subsec:comparison-yong-infmax}
We provide an instance of adaptive influence maximization such that 
the adaptive submodularity ratio yields an approximation ratio significantly better 
than that obtained with the weak adaptive submodularity~\citep{YGO17}. 

\begin{exmp}
We use the same problem instance as \Cref{exmp:infmax-star}.
At the beginning, the expected marginal gain of $v_k$ is $\Delta(v_k | \emptyset) = 1 / k$.
Let $\psi$ be the observations obtained when $v_1,\dots, v_{k-1}$ are selected and all edges are turned out to be dead.
In this case, since the edge $(v_k, u)$ must be alive, the expected marginal gain is $\Delta(v_k | \psi) = 1$.
The weak adaptive submodularity constant is lower-bounded as $\zeta \ge \Delta(v_k | \psi) / \Delta(v_k | \emptyset) = k$.
	This implies that the weak adaptive submodularity constant cannot yield a lower bound of the approximation ratio better than $1 - \exp(- \frac{1}{k}) = \rmO(\frac{1}{k})$, while the adaptive submodularity ratio provides a lower bound $1 - \exp(- (k + 1) / 2k) = \Omega(1)$.
\end{exmp}

\subsection{Comparison in Adaptive Feature Selection}\label{subsec:comparison-yong-feature}
	Regarding adaptive feature selection, we describe an advantage of the adaptive submodularity ratio 
	in comparison with the weak adaptive submodularity \cite{YGO17}. 
	As detailed below, there exists an instance with the following condition: 
	the approximation ratio obtained with the adaptive 
	submodularity ratio is bounded, 
	while that obtained with the weak adaptive submodularity is $0$. 
\begin{exmp}
	We can make such an instance even if $\phi$ is deterministic.
	Let $\bfA(\phi) = (\phi(1) \cdots \phi(n))$ be the realized feature matrix under realization $\phi$.
	The objective function is defined as 
	\begin{equation}
	f(S, \phi) = \| \bfb \|^2_2 - \min_{\bfw \in \bbR^S} \| \bfb - \bfA(\phi)_S \bfw \|^2_2.
	\end{equation}
	We here let 
	\begin{align}
	\bfA(\phi) 
	= 
	\begin{bmatrix}
	1 & 1/\sqrt{2} & 0 &  \cdots & 0\\
	0 & 1/\sqrt{2} & 0 &  \cdots & 0 \\
	0 & 0 & 1 &   &   &   \\
	\vdots & \vdots &    & \ddots &   \\
	0 & 0 &    &   &  1 \\
	\end{bmatrix}
	& &\text{and}& &
	\bfb
	= 
	\begin{bmatrix}
	0 \\
	a \\
	a \\
	\vdots \\
	a
	\end{bmatrix},  
	\end{align}
	where $a>0$ is an any positive real value. 
	Let $S=\{3,\dots,n\}$ and $T=\{2,\dots,n\}$, 
	which satisfy $S\subseteq T$. 
	Then, we have
	\begin{align}
	\min_{\bfw \in \bbR^S} \| \bfb - \bfA(\phi)_S \bfw \|^2_2
	=
	\min_{\bfw \in \bbR^{S\cup\{1\}}} \| \bfb - \bfA(\phi)_{S\cup\{1\}} \bfw \|^2_2
	=
	a^2
	\end{align}
	and 
	\begin{align}
	\min_{\bfw \in \bbR^T} \| \bfb - \bfA(\phi)_S \bfw \|^2_2
	=
	\frac{a^2}{2}
	>
	\min_{\bfw \in \bbR^{T\cup\{1\}}} \| \bfb - \bfA(\phi)_{T\cup\{1\}} \bfw \|^2_2
	=
	0.
	\end{align}
	Therefore, we obtain
	\begin{align}
	f(S\cup\{1\},\phi) - f(S,\phi) = a^2 - a^2 = 0 
	& & \text{and} & & 
	f(T\cup\{1\},\phi) - f(T,\phi) = 
	\frac{a^2}{2} - 0 
	=\frac{a^2}{2},  
	\end{align} 
	which implies that $\zeta$ cannot be bounded from above in general. 
	On the other hand, 
	the largest and smallest eigenvalues of the Hessian, $\bfA(\phi)^\top \bfA(\phi)$, 
	are $1+1/\sqrt{2}$ and $1-1/\sqrt{2}$, 
	respectively. 
	Therefore, 
	the condition number is bounded from above by 
	$3+2\sqrt{2}$, 
	which means the adaptive submodularity ratio is bounded 
	from below by $1/(3+2\sqrt{2})$.  
\end{exmp}


\end{document}